%% file: main.tex
\documentclass[lettersize,journal]{IEEEtran}
\usepackage{amsmath,amsfonts}
\usepackage{algorithmic}
\usepackage{algorithm}
\usepackage{array}
\usepackage[caption=false,font=normalsize,labelfont=sf,textfont=sf]{subfig}
\usepackage{textcomp}
\usepackage{stfloats}
\usepackage{url}
\usepackage{verbatim}
\usepackage{graphicx}

\usepackage[pagebackref,breaklinks,colorlinks, citecolor=nicegreen,bookmarks=false]{hyperref}
\usepackage{subcaption}
\usepackage{xcolor}   
\definecolor{grayish}{rgb}{0.95, 0.95, 0.95}
\definecolor{nicegreen}{rgb}{0.1, 0.7, 0.1}
\usepackage{mdframed}
\usepackage[table]{xcolor}
\usepackage{minibox}
\usepackage{svg}
\usepackage{amsmath}
\usepackage{amssymb}
\usepackage{mathtools}
\usepackage{amsthm}
\usepackage{booktabs} 
\usepackage{multirow} 
\usepackage{tcolorbox}
\usepackage[table]{xcolor}

\usepackage{cite}
\usepackage{cleveref}
\theoremstyle{plain}
\newtheorem{theorem}{Theorem}

\theoremstyle{definition}

\theoremstyle{remark}

\definecolor{LightBlue}{rgb}{0.88,1,0.88}
\definecolor{beaublue}{rgb}{0.9, 0.95, 0.9}
\definecolor{grayish}{rgb}{0.95, 0.95, 0.95}
\definecolor{blackish}{rgb}{0.2, 0.2, 0.2}
\definecolor{ao(english)}{rgb}{0.0, 0.5, 0.0}

\begin{document}

\title{Unifying Adversarially Robust Model Experts \\ in Vision-Language Models}

\author{Nguyen Duc Thai, Junhao Dong, Sua Qi Rong, Hua Yu, Tan Chin Sheng, Yew-Soon Ong, \textit{Fellow, IEEE} 
\thanks{Nguyen Duc Thai, Junhao Dong, Sua Qi Rong, Hua Yu, and Yew-Soon Ong are with the College of Computing and Data Science, Nanyang Technological University, Singapore. Yew-Soon Ong, and Junhao Dong are also with the Center for Frontier AI Research, Agency for Science, Technology and Research (A*STAR), Singapore (e-mail: \{ducthai001, suaq0001, junhao003, yu\_hua, asysong\}@ntu.edu.sg.)} 
\thanks{Tan Chin Sheng is with the Singapore Institute of Manufacturing Technology, Agency for Science, Technology and Research (A*STAR), Singapore (e-mail: cstan@simtech.a-star.edu.sg)}
}

\markboth{IEEE Transactions on Emerging Topics in Computational Intelligence}%
{Anonymous Authors: Unifying Adversarially Robust Model Experts in Vision-Language Models}


\maketitle

\begin{abstract}
Vision-language models (VLMs), such as CLIP, are vulnerable to adversarial attacks, posing a serious problem for real-life applications and deployment. Adversarial fine-tuning emerges as a prominent defense method; however, different fine-tuning strategies often produce specialized models with distinct robustness characteristics. Each fine-tuned model in turn thrives in some evaluation settings but falters on others, limiting its defensive capabilities. We refer to these specialized fine-tuned models as \emph{robust model experts} and propose a collaborative adversarial fine-tuning framework: \textbf{CARE} - \textbf{C}ollaborative \textbf{A}dversarial \textbf{R}obustness fine-tuning using \textbf{E}mbedding alignment. CARE maintains multiple experts during training, enables knowledge exchange through embedding-space harmonization, and consolidates the learned knowledge into a single unified robust model. Experts benefit from one another while preserving their individual specializations, enabling the final model to inherit complementary robustness properties. In this paper, we demonstrate CARE on two different adversarial fine-tuning strategies with complementary robustness behaviors. Extensive experiments on classic image classification and downstream vision-language tasks show the effectiveness of our approach, with CARE outperforming individually learned model experts. The results suggest that collaborative learning across model experts is a promising direction for improving adversarial robustness.
\end{abstract}

\begin{IEEEkeywords}
Adversarial training, In-distribution accuracy, Zero-shot accuracy, Collaborative learning.
\end{IEEEkeywords}

\section{Introduction}
\label{sec:introduction}
Vision-language models, such as CLIP~\cite{radford2021learning}, have achieved considerable success in many tasks, such as image captioning, text-to-image generation, or text-to-video generation. This success is due to their ability to align visual and textual representations in a shared embedding space. However, CLIP, like other multimodal foundation models, is vulnerable to adversarial attacks \cite{schlarmann2023adversarial}. Subtle, human-imperceptible perturbations can drastically alter predictions, posing serious security risks, such as performance drops, generation of harmful content, or even creating backdoors that attackers can exploit.


\begin{figure}[t]                    
  \centering
  \includegraphics[width=\linewidth]{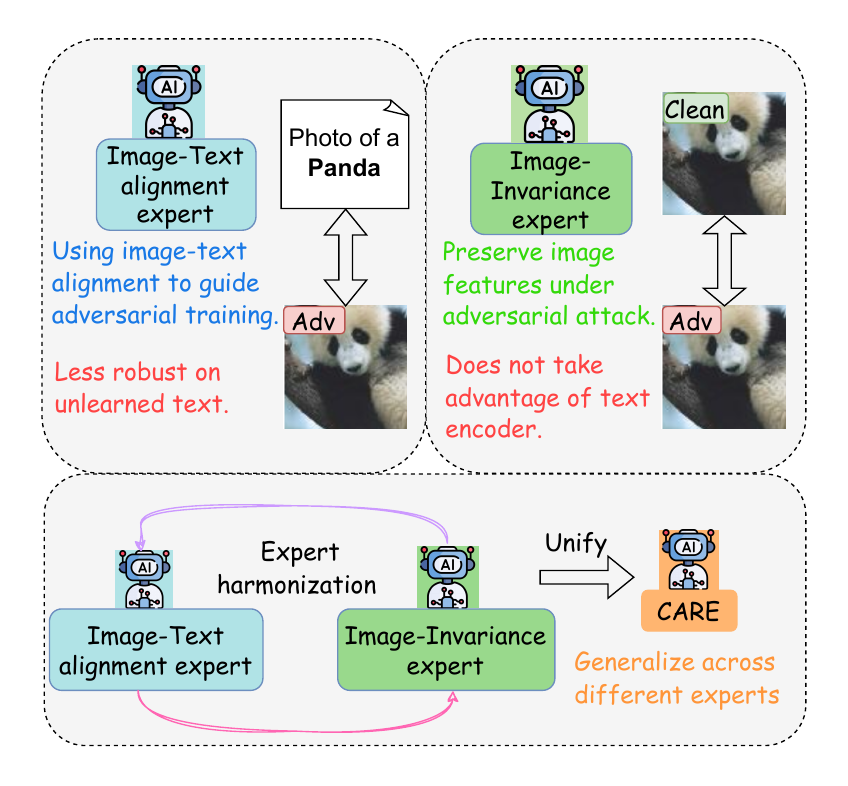}
  \caption{We identify two prominent approaches to adversarial fine-tuning for VLMs. The first relies on image-text alignment to guide the adversarial training process. This approach cannot guarantee alignment on untrained text, and therefore loses effectiveness on new classes. The second leverages VLMs' rich visual representations and teaches the model to stay consistent even under adversarial attack. While it is more general, it does not take advantage of the shared representation that is a hallmark of VLMs. Our framework, CARE, works to combine these different specialized experts. CARE harmonizes different experts, exchanging learned knowledge between them before unifying them into a single final model.}
  \label{fig:expert_combination}
\end{figure}

Therefore, adversarial robustness is essential for the safe and widespread deployment of VLMs. Among many adversarial defense methods, adversarial training remains one of the most effective strategies \cite{madry2017towards,zhang2019theoretically}. By training or fine-tuning on adversarial examples, the model becomes robust and invariant to adversarial attacks. The advances of VLMs bring about new defense opportunities for adversarial training.

First, VLMs perform image-text alignment in a shared embedding space. Thus, text embeddings and prompts can be used to guide the adversarial fine-tuning process. This provides a new defense surface not found in earlier CNNs. Many recent adversarial training methods apply this principle, such as TeCoA~\cite{mao2022understanding} or TGA~\cite{yu2024text}. TeCoA fine-tunes using image classification, with adversarial images created to maximize the cross-entropy loss from text and image alignment. TGA explores text-guided attention to improve robustness, but at its core the main idea is the same: relying on the shared image-text embedding space for guidance.

Second, the larger size and compute power of VLMs, compared to older CNNs, mean that their visual representations are much richer and they can be adversarially trained on proxy tasks, not necessarily image classification. A prominent example is FARE~\cite{schlarmann2024robust}: adversarial images are created to maximize distortion in the image embedding space. The model is trained without image-text supervision, with the sole aim of preserving visual representation, thus bolstering the invariance of image embeddings under attack.

We refer to an adversarially trained model as a robust model expert. Depending on the training angle they take, fine-tuned models exhibit different characteristics and their performance varies under different evaluation settings. Models that follow the image-text route often perform better on learned class names, since they rely on the text prompts for robustness. Given a new class, they struggle to maintain robustness. Meanwhile, models that focus on improving the image-invariance aspect can transfer their robustness to unseen datasets. However, since the image-text alignment is not directly maintained during training, their performance on known classes is lackluster compared to image-text experts. Furthermore, they also do not make use of the trademark feature of VLMs: the shared multimodal representation between image and text.

How can we generalize across different adversarial robust model experts? A direct solution would be to train independent experts, then combine them afterward. However, since different experts are optimized with different strategies, they naturally drift apart during training, and unifying them is unstable. Works that explore combining different models for adversarial training, such as \cite{wang2023generalist} and \cite{dong2024adversarially}, all run into this roadblock. These works therefore have to periodically overwrite the experts' parameters completely using a shared checkpoint to maintain stability. We find this approach limits the exploration and specialization of the experts. To overcome this, we propose CARE: \textbf{C}ollaborative \textbf{A}dversarial \textbf{R}obustness fine-tuning using \textbf{E}mbedding alignment. CARE trains multiple experts simultaneously. Each expert follows its specialization, while a harmonization mechanism is used to exchange knowledge, thus preventing drift and improving stability. The learned robustness is then unified into a single model, thereby generalizing across different adversarial training strategies.

\IEEEpubidadjcol

The contributions of this work are as follows:

\begin{itemize}
    \item We propose the CARE framework for learning across adversarial robust model experts. CARE maintains multiple experts, each derived from a different adversarial fine-tuning strategy, and unifies their learned knowledge into a single unified robust model.
    \item We introduce an embedding-space harmonization mechanism that enables knowledge exchange among experts during training, thus improving stability without sacrificing exploration.
    \item Extensive experiments demonstrate that CARE outperforms individually trained experts in both image classification and downstream vision-language tasks.
\end{itemize}

\section{Related Work}


Vision-language models, such as CLIP \cite{radford2021learning}, ALIGN \cite{jia2021scaling}, and ALBEF \cite{li2021align}, are trained on large image-text datasets. Each modality has its own encoder, and the model aligns the multimodal embeddings within a shared latent space to provide strong image-text representations. Harnessing the expressivity of these representations, vision-language models have achieved considerable success in multimodal tasks, such as zero-shot image classification, image-text retrieval, and image captioning or video captioning. In addition, many Large Vision-Language Models (LVLMs), such as Flamingo \cite{alayrac2022flamingo}, OpenFlamingo \cite{awadalla2023openflamingo}, LLaVA \cite{liu2023visual}, etc., employ them as vision encoders to handle image-related tasks.

However, VLMs are not immune to adversarial attack. Multiple works have explored the vulnerability of LVLMs and VLMs to attacks \cite{qi2023visual,carlini2023aligned,zhao2023evaluating,gu2024agent,ding2026feedback,chaturvedi2020attacking}. Among existing defense strategies, adversarial training is one of the most effective. Existing methods often fall into one of two categories: leveraging VLMs' image-text latent space or relying on their large architecture and rich visual representations for robustness.

The first category is best represented by the work of Mao et al.~\cite{mao2022understanding}. The authors investigated TeCoA, which leverages the strong image-text alignment of CLIP and uses the text encoder to guide the fine-tuning process. Other works, such as PMG~\cite{wang2024pre} or TGA~\cite{yu2024text}, add losses such as image-text alignment within the non-fine-tuned model or text-guided attention loss, but the main principle remains: they target the image-text alignment within CLIP for adversarial training. Schlarmann et al.~\cite{schlarmann2024robust} exemplified the second strategy. They proposed FARE, which adversarially trains CLIP by minimizing the distance between image embeddings of clean and adversarial samples. While the training does not directly address adversarial image classification, the strong image embeddings gained after fine-tuning can be transferred to a multitude of tasks. As stated in Section~\ref{sec:introduction}, we investigate these two different philosophies.

We choose TeCoA and FARE as the bases to demonstrate our framework, as we consider them representative of the two approaches. We notice that in image-text alignment experts, the logits between image and text embeddings are the primary target for adversarial fine-tuning, while for image-invariance experts, it is done directly through the image encoder, without any image-text-focused objectives. This insight motivates us to derive the harmonization mechanism of CARE: exchanging adversarial knowledge of adversarially fine-tuned embeddings.

Our approach bears resemblance to generalization through model merging \cite{wortsman2022model,ilharco2022editing,ortiz2023task}. Pretrained models are combined in parameter space to improve performance, balance trade-offs, etc., without any increase in inference cost. In the realm of adversarial training, Wang et al.~\cite{wang2023generalist} proposed an adversarial training framework to balance the clean-robustness trade-off. Each task is handled by a dedicated expert: one focuses on clean accuracy while the other aims to improve robustness; their parameters are later merged to form a unified model that aggregates knowledge from all tasks. However, the authors revealed that due to differing specializations, the merging process can become unstable, leading to worse results. Their solution is to perform parameter overwriting, where the parameters of the unified model are copied on top of each expert. Instead, we use our knowledge-exchange process to harmonize the experts and prevent divergence.

\section{Collaborative Adversarial Fine-tuning}

In this section we introduce our adversarial fine-tuning framework CARE. The framework performs adversarial fine-tuning across robust model experts and maintains harmonization through embedding alignment. An overview of the method is presented in Figure~\ref{fig:care_method}.

\begin{figure*}[t]                    
  \centering
  \includegraphics[width=\linewidth]{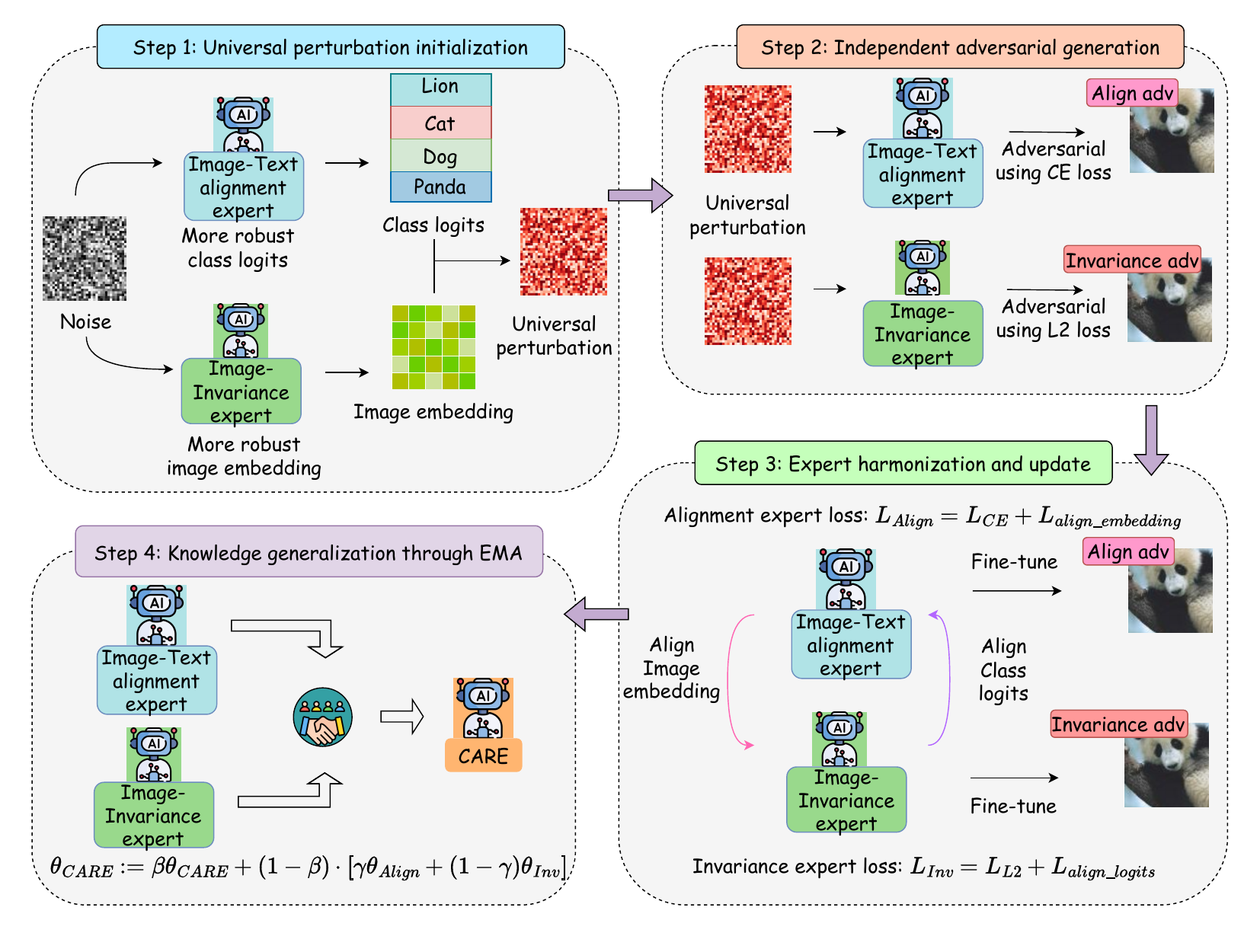}
  \caption{Our method CARE consists of 4 steps: initialize a universal perturbation with inputs from both experts; each expert generates its own adversarial samples; experts optimize their parameters using joint losses: adversarial fine-tuning loss for specialization and harmonization loss to learn from each other; finally, the parameters are combined using the Exponential Moving Average (EMA) method. The details for each expert's loss can be found in Eq.~\eqref{eq:ce_branch_loss} and Eq.~\eqref{eq:l2_branch_loss}.}
  \label{fig:care_method}
\end{figure*}

\subsection{Preliminaries}
\label{sub:method_premilinaries}

CLIP, introduced by Radford et al.~\cite{radford2021learning}, is a breakthrough in vision-language pretrained models. CLIP consists of two encoders: an image encoder \(f_{\theta}: X \to \mathbb{R}^{d}\) and a text encoder \(f_{\phi}: T \to \mathbb{R}^{d}\), sharing the same dimensionality $d$. Here $X$ and $T$ represent the input spaces for images and text, and $\theta$ and $\phi$ represent the model parameters for the image and text encoder, respectively.
Zero-shot classification of an image $x$ with $C$ output classes is done by generating the set of class prompts $\mathbf{t} = \{{t}_{1}, {t}_{2}, \dots, {t}_{C}\}$ describing the image, i.e., ${t}_{c} = $ ``\texttt{This is a photo of a <class c>}''. Cosine similarity between the image embedding of $x$ and the text embedding of $t_c$ is used to compute each class logit.
\begin{align}
    z_{c}(x,\theta, \phi) &= \cos(f_{\theta}(\mathbf{x}), f_{\phi}(\mathbf{t_{c}})) \quad \text{for } c=1,\ldots,C.
\end{align}
Here $z_{c}(x,\theta, \phi)$ is the logit for class $c$ and the logit vector is $\mathbf{z(x,\theta, \phi)}=[z_{1}(x,\theta, \phi),\ldots,z_{C}(x,\theta, \phi)]$. We can then compute the class probability using the softmax function.
\begin{align}
    p_{c}(x, \theta, \phi) &= \frac{\exp \left(z_c\right)}{\sum_{c'=1}^{C} \exp \left(z_{c'}\right)}.
\end{align}
$\mathbf{p_{\theta,\phi}(x)} = [p_1(x,\theta,\phi),\ldots,p_C(x,\theta,\phi)]$ is the prediction vector. Given the ground-truth label $y$ of $x$, the aim of an adversarial attack in the $\ell_p$-norm threat model is then to generate an adversarial sample $\tilde{x}$ such that the model misclassifies $\tilde{x}$ into another class:
\begin{align}
    \arg\max_{c=1,\ldots,C} \mathbf{p_{\theta,\phi}(\tilde{x})} \neq y, \| \mathbf{\tilde{x}} - \mathbf{x} \|_p \leq \epsilon, \mathbf{\tilde{x}} \in X ,
\end{align}
where $\epsilon$ is the perturbation radius. Following prior works, we focus on $\ell_\infty$-norm threats.

Mao et al.~\cite{mao2022understanding} introduced TeCoA, a text-guided adversarial fine-tuning approach for CLIP on ImageNet. TeCoA minimizes the cross-entropy between the image-text logits of the adversarial samples and the true labels: $L_{CE}( z(\tilde{x},\theta,\phi), y)$. The robust training of TeCoA can be written as a minimax optimization problem:
\begin{align}
    \theta_{TeCoA} &= \arg\min_{\theta}\max_{\|\tilde{x} - x\|_{\infty} \leq \epsilon} L_{CE}(z(\tilde{x},\theta,\phi),y).
\end{align}

Schlarmann et al.~\cite{schlarmann2024robust} suggested an unsupervised approach to robust fine-tuning using $\ell_2$-distance loss, anchoring the adversarial image embeddings to those of the original clean embeddings, preventing drift in image feature space:
\begin{align}
    L_{L2}(f_{\theta}(\tilde{x}), f_{\theta_{Org}}(x)) = \|f_{\theta}(\tilde{x}) - f_{\theta_{Org}}(x)\|_{2}^2.
\end{align}
FARE enforces image feature invariance between clean and adversarial samples, therefore promoting robustness. The robust formulation of FARE is the following:
\begin{align}
    \theta_{FARE} &= \arg\min_{\theta}\max_{\|\tilde{x} - x\|_{\infty} \leq \epsilon} L_{L2}(f_{\theta}(\tilde{x}), f_{\theta_{Org}}(x)).
\end{align}

In both TeCoA and FARE, the inner maximization problem is approximated using Projected Gradient Descent (PGD) \cite{madry2017towards}. We denote the cross-entropy loss and the $\ell_2$-distance loss as $L_{CE}(\tilde{x}, y)$ and $L_{L2}(\tilde{x}, x)$. Since both FARE and TeCoA optimize the image encoder only, we simplify the notation for logits and class probability into $z(x, \theta)$ and $p(x, \theta)$.

We choose these two to be the bases for our framework due to their isolated focus on different components in the CLIP architecture: TeCoA uses text features to improve robustness, while FARE uses image features. Our framework, CARE, maintains separate networks with different specializations: image-text alignment expert, based on TeCoA, or \textbf{Alignment expert} for short; and image-invariance expert, based on FARE, or \textbf{Invariance expert} for short. The Alignment expert utilizes cross-entropy loss for adversarial training, while the Invariance expert relies on $\ell_2$-distance loss. In this work we focus on adversarial training for the image encoder only, and we denote the corresponding parameters of these two image encoders as $\theta_{Align}$ and $\theta_{Inv}$, respectively. A final network CARE, with parameters $\theta_{CARE}$, is produced by unifying the two experts.

The three networks $\theta_{Align}$, $\theta_{Inv}$, and $\theta_{CARE}$ are initialized using the same pretrained CLIP image encoder. $\theta_{Align}$ and $\theta_{Inv}$ are optimized directly during training. Both experts utilize the same clean data sample to generate their distinct adversarial samples using $K$-step PGD. We apply Exponential Moving Average (EMA) after each training step to unify learned knowledge from the two:
\begin{align}
    \theta_{CARE} := \beta \theta_{CARE} + (1-\beta)[\gamma\theta_{Align} + (1-\gamma)\theta_{Inv}].
\end{align}
Here, $0 \leq \beta \leq 1$ is the decay rate and $0 \leq \gamma \leq 1$ controls the influence of each expert.

However, because of their different specializations, the experts tend to diverge, leading to an unstable combination. We propose an expert harmonization method to allow robust optimization while maintaining synchronization.

\subsection{Expert harmonization}

Our expert harmonization mechanism aligns each expert to implicitly learn the adversarial fine-tuning process of the other. We observe that $\theta_{Align}$ optimizes the cross-entropy loss by acting directly on the image-text logits. Meanwhile, $\theta_{Inv}$ optimizes the L2 loss by operating on the image embeddings. We add an embedding alignment loss to each expert: $\theta_{Inv}$ will align its clean logits (which it does not optimize) to those of $\theta_{Align}$ (which optimizes logits directly). Likewise, the $\theta_{Align}$ network will align its clean image embeddings to those of $\theta_{Inv}$. This cross-learning mechanism enhances the robustness of each expert using both the image and text encoders and promotes consistency between their optimization dynamics.

Given an input image $x$ with label $y$, we denote by $\tilde{x}_{Align}$ and $\tilde{x}_{Inv}$ the adversarial samples that the experts generate independently by maximizing the cross-entropy loss and the $\ell_2$-distance loss as described in Section~\ref{sub:method_premilinaries}. $L_{align\_z}$ is the loss for logit alignment, used by the Invariance expert. $L_{align\_f}$ is the embedding alignment loss, used by the Alignment expert. We use cosine similarity loss by minimizing the negative cosine similarity for alignment losses.
\begin{equation}
    L_{align\_z}(x, \theta_{Align}, \theta_{Inv}) = - \cos(z(x, \theta_{Inv}), z(x, \theta_{Align})),
\end{equation}
\begin{equation}
    \begin{split}
        L_{Inv} = L_{L2}(f_{\theta_{Inv}}(\tilde{x}_{Inv}), f_{\theta_{Org}}(x))
        \\ \quad + L_{align\_z}(x, \theta_{Align}, \theta_{Inv})    ,
    \end{split}
    \label{eq:l2_branch_loss}
\end{equation}
\begin{equation}
    L_{align\_f}(x, \theta_{Align}, \theta_{Inv}) = - \cos(f_{\theta_{Align}}(x), f_{\theta_{Inv}}(x)),
\end{equation}
\begin{equation}
    \begin{split}
        L_{Align} = L_{CE}(z(\tilde{x}_{Align}, \theta_{Align}), y)
        \\ \quad + L_{align\_f}(x, \theta_{Align}, \theta_{Inv})   .
    \end{split}
    \label{eq:ce_branch_loss}
\end{equation}

\begin{theorem}
\label{thm:branch_momentum}
Assume that $L_{align\_z}(x, \theta_{Align}, \theta_{Inv})$ and $L_{align\_f}(x, \theta_{Align}, \theta_{Inv})$ are twice differentiable in the neighborhood of $\theta$. Let $H_{z}(\theta):= \nabla^{2}_{\theta_{Inv}}L_{align\_z}(x, \theta_{Align}, \theta_{Inv})$ and $H_{f}(\theta_{Align}) := \nabla^{2}_{\theta}L_{align\_f}(x, \theta_{Align}, \theta_{Inv})$, for a sufficiently small drift between the two model parameters $\theta_{Align}$ and $\theta_{Inv}$, the auxiliary losses act as momentum pulling the experts toward each other, where the gradient updates are:
\begin{equation}
    \nabla_{\theta}L_{align\_z}(x, \theta_{Align}, \theta_{Inv}) = H_{z}(\theta_{Align})(\theta_{Inv} - \theta_{Align}),
\end{equation}
\begin{equation}
    \nabla_{\theta}L_{align\_f}(x, \theta_{Align}, \theta_{Inv}) = H_{f}(\theta_{Inv})(\theta_{Align} - \theta_{Inv}).
\end{equation}
\end{theorem}

\begin{proof}
Let $\Delta \theta = \theta_{Inv} - \theta_{Align}$ be the drift between the two experts, then $\theta_{Inv} = \theta_{Align} + \Delta \theta$, and $\theta_{Align} = \theta_{Inv} - \Delta \theta$.

Consider the auxiliary loss in $L_{Inv}$ (Eq.~\eqref{eq:l2_branch_loss}):

\begin{center}
    $L_{align\_z}(x, \theta_{Align}, \theta_{Inv}) = - \cos(z(x, \theta_{Inv}), z(x, \theta_{Align}))$,
\end{center}

where we align the logits of the Invariance expert with those of the Alignment expert. The reference logits, computed with $\theta_{Align}$, are frozen, therefore the loss is a function of $\theta_{Inv}$ only.

Perform second-order Taylor expansion around $\theta_{Align}$:
\begin{align}
    L_{align\_z} &= L_{align\_z}(x, \theta_{Align}, \theta_{Align} + \Delta \theta) \notag \\
    &= -\cos(z(x, \theta_{Align}), z(x, \theta_{Align} + \Delta \theta)) \notag \\
    &= -\cos(z(x, \theta_{Align}), z(x, \theta_{Align})) + \notag \\ & \nabla_{\theta}(-\cos(z(x, \theta_{Align}), z(x, \theta_{Align})))\Delta \theta + \notag \\
    & \frac{1}{2}\Delta \theta^T H_{z}(\theta)\Delta \theta + \mathcal{O}(\|\Delta\theta\|^{3}).
\end{align}

Here, $H_z$ represents the Hessian matrix through the logits cosine similarity. Because we calculate the cosine similarity of the logits at $\theta_{Align}$ with its reference logits also at $\theta_{Align}$, the cosine similarity becomes 1, and its first-order derivative is 0, simplifying the expansion to:
\begin{align}
    L_{align\_z} &= -1 + \frac{1}{2}\Delta \theta^T H_{z}(\theta_{Align})\Delta \theta + \mathcal{O}(\|\Delta\theta\|^{3}).
\end{align}

For sufficiently small $\|\Delta\theta\|$, $\mathcal{O}(\|\Delta\theta\|^{3})$ becomes negligible, leaving only $\frac{1}{2}\Delta \theta^T H_{z}(\theta)\Delta \theta$.

Thus, the gradient direction for the auxiliary loss is:
\begin{align}
    \nabla_{\theta}L_{align\_z}(x, \theta_{Align}, \theta_{Inv}) &\approx H_{z}(\theta_{Align})\Delta\theta \notag \\
    &= H_{z}(\theta_{Align})(\theta_{Inv} - \theta_{Align}).
\end{align}

Similarly, substituting $\theta_{Align} = \theta_{Inv} - \Delta\theta$ and performing a similar expansion on the auxiliary loss for the Alignment expert gives (flipped sign due to subtraction of $\Delta\theta$):
\begin{align}
    \nabla_{\theta}L_{align\_f}(x, \theta_{Align}, \theta_{Inv}) &\approx -H_{f}(\theta_{Inv})\Delta\theta \notag \\
    &= -H_{f}(\theta_{Inv})(\theta_{Inv} - \theta_{Align}) \notag \\
    &= H_{f}(\theta_{Inv})(\theta_{Align} - \theta_{Inv}).
\end{align}

Similar to the logits alignment loss, $H_{f}(\theta_{Inv})$ is the Hessian matrix through the feature cosine similarity.
\end{proof}

\begin{tcolorbox}[width=1.0\linewidth, colframe=blackish, colback=beaublue, boxsep=0mm, arc=2mm, left=2mm, right=2mm, top=5mm, bottom=2mm]
\vspace{-0.3cm}
In one step, the update to the Invariance expert would be equal to $-\eta \nabla_{\theta}L_{align\_z} = -\eta H_{z}(\theta_{Align})(\theta_{Inv} - \theta_{Align})$, penalizing the Invariance expert if it is far away from the Alignment expert. Vice versa, the Alignment expert is pulled toward the last known location of the Invariance expert. Therefore, in each update step, the experts are prevented from moving too far from each other.
\end{tcolorbox}

Theorem~\ref{thm:branch_momentum} motivates our design choices for CARE. Namely, both cosine similarity and clean representation alignment (we align the clean logits and clean embeddings) are chosen because in this way we can omit the first-order derivative in the Taylor expansion, making theoretical analysis easier. Intuitively, the auxiliary losses force the experts to maintain the same clean representation while exploration comes from different adversarial specialization. The drift  $\Delta \theta$ is small at the beginning of training, as the experts are initialized from the same pretrained network, and the momentum from the auxiliary losses keeps them from diverging.

\begin{algorithm}[!t]
\caption{\textbf{C}ollaborative \textbf{A}dversarial \textbf{R}obustness fine-tuning using \textbf{E}mbedding alignment (CARE)}
\label{alg:CARE}
{\small
\textbf{Input}:
\begin{itemize}
    \item Parameter sets: pre-trained CLIP $\theta_{CLIP}$; Alignment expert $\theta_{Align}$; Invariance expert $\theta_{Inv}$; unified model $\theta_{CARE}$.
    \item $T$ total fine-tuning steps, $M$ collaborative perturbation initialization steps, $K$ adversarial generation steps, and adversarial generation step size $\alpha$. Perturbation radius $\epsilon$.
    \item Dataset $D$, learning rate $\eta$.
    \item EMA decay rate $\beta$ and expert weight $\gamma$.
\end{itemize}
}
{\small
\begin{algorithmic}[1]
    \STATE Initialize $\theta_{Align}$, $\theta_{Inv}$, $\theta_{CARE}$ $\leftarrow$ $\theta_{CLIP}$
    \FOR {$t=1, \ldots, T$}
        \STATE Sample $(x,y)$ from dataset $D$
        \STATE $\tilde{x}^{0} \leftarrow x + \delta, \delta \sim \mathcal{U}(-\epsilon, \epsilon)$ 

        \FOR {$m = 1,\ldots,M$} 
            \STATE \colorbox{grayish}{ $\tilde{x}^{m} = \arg\max_{\|\tilde{x}-x\|_{\infty} \leq \epsilon}L_{GAIP}(\tilde{x}^{m-1}, x)$.}
        \ENDFOR

        \STATE$\tilde{x}_{Align}^{0}, \tilde{x}_{Inv}^{0} \leftarrow \tilde{x}^{M}$

        \FOR{$k = 1,\ldots,K$}
            \STATE \colorbox{grayish}{$\tilde{x}_{Align}^{k} = \arg\max_{\|\tilde{x}-x\|_{\infty} \leq \epsilon}L_{CE}(\tilde{x}_{Align}^{k-1}, y)$}
            \STATE \colorbox{grayish}{$\tilde{x}_{Inv}^{k} = \arg\max_{\|\tilde{x}-x\|_{\infty} \leq \epsilon}L_{L2}(\tilde{x}_{Inv}^{k-1}, x)$}
        \ENDFOR

        \STATE \fcolorbox{white}{grayish}{Parameter update for each expert:}
            \STATE \qquad $\theta_{Align} \leftarrow \theta_{Align} - \eta \nabla_{\theta_{Align}}L_{Align}$ \COMMENT{Eq.~\eqref{eq:ce_branch_loss}}

            \STATE \qquad $\theta_{Inv} \leftarrow \theta_{Inv} - \eta \nabla_{\theta_{Inv}}L_{Inv}$\COMMENT{Eq.~\eqref{eq:l2_branch_loss}}

        \STATE \colorbox{grayish}{\minibox{Co-distillation using EMA: \\
        $\theta_{CARE} := \beta \theta_{CARE} + (1-\beta)[\gamma\theta_{Align} + (1-\gamma)\theta_{Inv}] $}}
    \ENDFOR
    \STATE {\bfseries return} Fine-tuned model parameters $\theta_{CARE}$
\end{algorithmic}
}
\end{algorithm}

\subsection{Universal perturbation initialization}
\label{sec:universal_perturbation}
To further improve the coherence between the two experts, we take inspiration from \cite{dong2024adversarially,moosavi2017universal,shafahi2020universal} and design our own \textit{\textbf{G}lobal \textbf{A}dversarial perturbation \textbf{I}nitialization \textbf{P}rocess} (GAIP). Instead of sampling independent random noise for each expert, GAIP constructs a shared universal perturbation that benefits both adversarial specializations.

We use $\theta_{Align}$ and $\theta_{Inv}$ to generate this initial perturbation. During this process, they are frozen. We align the perturbed logits with those of $\theta_{Align}$ and its image embeddings with those of $\theta_{Inv}$ using cosine distance. The loss for generating the universal perturbation is:
\begin{equation}
    h(\tilde{x}) = 1 - \cos\big(z(\tilde{x}, \theta_{Align}), z(x, \theta_{Align})\big),
    \label{eq:gaip_ce}
\end{equation}
\begin{equation}
    g(\tilde{x}) = 1 - \cos\big(f_{\theta_{Inv}}(\tilde{x}), f_{\theta_{Inv}}(x)\big),
    \label{eq:gaip_l2}
\end{equation}
\begin{equation}
    \begin{split}
        L_{GAIP}(\tilde{x},x) = h(\tilde{x}) + g(\tilde{x}).
    \end{split}
    \label{eq:loss_gaip}
\end{equation}
where $h(\tilde{x})$ and $g(\tilde{x})$ are losses for alignment of logits and image embeddings, respectively.

Intuitively, GAIP follows the principle of expert harmonization, where experts follow the stronger embeddings of each other. We find that instead of applying the expert losses (CE and L2) directly, using cosine similarity produces a better result and is easier to control, as CE and L2 losses operate on different scales. More details can be found in Section~\ref{sec:gaip_ablation}.

The GAIP perturbation is generated through an $M$-step PGD process using the GAIP loss described above. Then, each expert will use that initialized perturbation to generate its own adversarial sample through a $K$-step PGD: the Alignment expert with CE loss and the Invariance expert with L2 loss.

\subsection{Pseudocode}
\label{appendix:pseudocode}

In this section, we present the algorithmic specification of CARE. The algorithm consists of 4 main steps, as shown in Figure~\ref{fig:care_method}: first, the universal perturbation is initialized; second, each expert uses the perturbation to generate an adversarial sample befitting its optimization process (image-text feature alignment or close clean-adversarial image features); third, the experts are optimized using a joint loss of usual robustness optimization and expert harmonization loss; finally, they combine their parameters using EMA to produce the unified model. The algorithm is shown in Algorithm~\ref{alg:CARE}.

\input{tables/baselines_b32}
\input{tables/zeroshot}

\section{Results}
\label{sec:results}
We find that image-text experts like TeCoA thrive under learned vocabulary where the image-text alignment is best maintained, whereas image-invariance experts like FARE perform better when facing new classes not seen during training. In this section, we compare CARE against the following methods on ViT-B/32: TeCoA, FARE, as well as the text-guided methods PMG and TGA. All text-guided approaches achieve similar results, so for the evaluation on ViT-L/14 with downstream tasks, we compare CARE only against TeCoA and FARE, highlighting its effectiveness in generalizing between different robust model experts. We also conduct investigations on the effect of the expert harmonization method to show CARE's ability to prevent expert divergence.

\noindent\textbf{Settings.} We use the validation set of ImageNet-1K \cite{deng2009imagenet} as the in-distribution (seen dataset) benchmark. Thirteen additional zero-shot datasets are used to measure zero-shot accuracy on unseen data. We also test our model on two vision-language tasks: image captioning and visual question answering. Datasets and evaluation protocols for each task are provided in their specific sections below.

\noindent \textbf{Implementation details.} Unless otherwise stated, all results are from CLIP ViT-L/14 variants. Following \cite{mao2022understanding,schlarmann2024robust}, training involves 10-step PGD ($\ell_\infty$, $\epsilon \in \{\tfrac{2}{255}, \tfrac{4}{255}\}$, $\alpha = \tfrac{1}{255}$). Following \cite{schlarmann2024robust}, we use superscripts to denote the $\epsilon$ value used for training, e.g., CARE\textsuperscript{2} is trained with $\epsilon=\frac{2}{255}$ and CARE\textsuperscript{4} with $\epsilon = \frac{4}{255}$. For fairness, as our GAIP process also employs PGD, we use 5-step PGD for universal perturbation generation and 5 steps for adversarial generation. TeCoA and FARE use all 10 steps for adversarial generation. For EMA merging, we set the decay rate $\beta=0.999$ and the expert weight $\gamma=0.5$. We use AutoAttack \cite{croce2020reliable} for zero-shot robustness evaluation and adaptive attacks for downstream vision-language tasks. For fine-tuning details, we use AdamW \cite{loshchilovdecoupled} as the optimizer ($\beta_1=0.9$, $\beta_2=0.95$). The learning rate is set to $5\!\times\!10^{-5}$ and weight decay is set to $1\!\times\!10^{-4}$. The batch size is 128, and the fine-tuning is done on 4 NVIDIA A100 80GB GPUs.

\input{tables/image_captioning}
\input{tables/vqa}

\subsection{Image classification}

\noindent \textbf{Datasets.} Following previous works on VLMs' robustness \cite{mao2022understanding,schlarmann2024robust}, we train our models on the training split of ImageNet-1K \cite{deng2009imagenet}. Zero-shot classification is reported for the validation split of ImageNet-1K (test labels are unavailable), and this is used to test the in-distribution performance. In addition, performance on unseen data is evaluated using the 13 datasets covering a wide range of scenarios:

\begin{itemize}
    \item \textbf{Natural objects:} STL-10 \cite{coates2011analysis}, CIFAR-10/100 \cite{krizhevsky2009learning}, Caltech-101 \cite{fei2004learning}.
    \item \textbf{Fine-grained:} Stanford Cars \cite{krause20133d}, Oxford-IIIT Pets \cite{parkhi2012cats}, Flowers-102 \cite{nilsback2008automated}, FGVC-Aircraft \cite{maji2013fine}.
    \item \textbf{Textures:} DTD \cite{cimpoi2014describing}.
    \item \textbf{Remote sensing:} EuroSAT \cite{helber2019eurosat}.
    \item \textbf{Medical:} PCAM \cite{veeling2018rotation}.
    \item \textbf{Robustness variants of ImageNet:} ImageNet-R \cite{hendrycks2021many} and ImageNet-S \cite{wang2019learning}.
\end{itemize}

\noindent\textbf{Evaluation protocol.} We report clean accuracy and robust accuracy under AutoAttack \cite{croce2020reliable}. For ViT-B/32 evaluation, we assess accuracy under $\epsilon = \tfrac{4}{255}$. The ViT-L/14 evaluations are done on both $\epsilon = \tfrac{2}{255}$ and $\epsilon = \tfrac{4}{255}$. Robust accuracy is evaluated on 1000 images of each dataset, and clean accuracy on all samples. To show the different specializations of robust model experts and how CARE generalizes across them, we report the results of ImageNet-1K and the other 13 datasets.

\noindent\textbf{Comparison with text-guided methods on ViT-B/32.} PMG~\cite{wang2024pre} and TGA~\cite{yu2024text} are both built upon TeCoA, adding further image-text alignment objectives for robustness. In their original papers, both methods are trained and evaluated on ViT-B/32, using Tiny-ImageNet with $\epsilon=\tfrac{1}{255}$. In order to compare them with CARE, we retrain TeCoA, FARE, PMG, TGA, and CARE on ViT-B/32 following our setting (ImageNet-1K, $\epsilon=\tfrac{4}{255}$, 10-step PGD), and evaluate them on ImageNet and the same 13 zero-shot datasets. Results are shown in Table~\ref{tab:baselines_b32}. Overall, CARE achieves the most balanced results with a clear gap over all others on clean data. Notably, all text-guided methods (TeCoA, PMG, TGA) see a big gap when it comes to unseen datasets, with a drop of over 8\% in clean accuracy. Under attack, CARE has the best zero-shot robust accuracy (25.6\% against 24.0\% for PMG), while TeCoA, PMG, and TGA perform slightly better on ImageNet. This matches our observation from Section~\ref{sec:introduction}: PMG and TGA are image-text alignment experts, and thus they thrive on learned class names. In the end, the overall robust accuracy of CARE, TeCoA, PMG, and TGA is close, within 0.5 percentage points. We find that at this scale, CARE trades some in-distribution robustness for better clean accuracy and robustness on unseen datasets. Since PMG and TGA follow the same principle as TeCoA and achieve similar results, we only compare CARE against TeCoA and FARE on ViT-L/14. PMG and TGA can also be used as the Alignment expert in CARE, which we leave for future work. Next, we discuss the results on ViT-L/14.

Details on image classification for ViT-L/14 are recorded in Table~\ref{tab:zero_shot_results}. CARE shows a noticeable increase in accuracy across almost all settings, both on ImageNet and on other datasets. Under the strongest attack setting, with $\epsilon=\tfrac{4}{255}$, CARE outperforms TeCoA by almost 3 and FARE by nearly 8 percentage points. We reiterate our previous point from Section~\ref{sec:introduction}, that image-text alignment experts are stronger on learned class names and image-invariance experts can better transfer robustness to unseen datasets. Considering the preferred evaluation dataset for each expert, ImageNet for TeCoA and zero-shot datasets for FARE, CARE still outperforms both experts with a clear improvement.

For models trained on $\epsilon=\tfrac{2}{255}$, that is, TeCoA$^2$, FARE$^2$ and CARE$^2$, we observe another phenomenon. As we increase attack strength, FARE accuracy drops rapidly, from 70.6\% on clean to 44.6\% for $\epsilon=\tfrac{2}{255}$, and falls to 18.6\% on $\epsilon=\tfrac{4}{255}$---the attack that models are not trained for. However, both TeCoA and CARE show a less sharp decline. This suggests that image-text alignment experts, such as TeCoA, can be more consistent and robust on image classification. It also shows that CARE is able to inherit this strong trait.

\subsection{Downstream vision-language tasks}
\noindent \textbf{Datasets.} For downstream vision-language tasks, we evaluate image captioning and visual question answering. Image captioning uses COCO \cite{lin2014microsoft} and Flickr30k \cite{plummer2015flickr30k}. Visual question answering uses TextVQA \cite{singh2019towards} and VQAv2 \cite{goyal2017making}.

\noindent\textbf{Evaluation protocol.} Evaluation on vision-language tasks is done by replacing the ViT-L/14 image encoder in LLaVA 1.5-7B \cite{liu2023visual} with CARE. All other components remain fixed. We experiment on image captioning and visual question answering. In both cases, 500 images are randomly sampled for adversarial evaluation, and all images are used for clean evaluation. The details can be found below.

\noindent\textbf{Image captioning.} CIDEr score \cite{vedantam2015cider} is reported as the evaluation metric. Adversarial samples are generated using APGD \cite{croce2020reliable} at 100 iterations. After each attack, an early-stop mechanism is used to remove all samples with a CIDEr score below 10 for COCO or below 2 for Flickr30k. The $\epsilon$ used for APGD is set to be the same $\epsilon$ the image encoders are trained on, $\tfrac{2}{255}$ and $\tfrac{4}{255}$. Results are shown in Table~\ref{tab:image_captioning_results}.

\noindent\textbf{Visual question answering.} We assess VQA accuracy \cite{antol2015vqa} for the TextVQA \cite{singh2019towards} and VQAv2 \cite{goyal2017making} datasets, selecting the five most frequent answers among the ten ground-truth annotations for each sample. We also employ APGD \cite{croce2020reliable} as the attacking method, with $\epsilon \in \{\tfrac{2}{255}, \tfrac{4}{255}\}$, in this case with targeted perturbations using target strings such as ``\texttt{Maybe}'' or ``\texttt{Word}'', following \cite{schlarmann2024robust}. Results are shown in Table~\ref{tab:vqa_results}.

For downstream tasks, CARE also showcases its advantage as the most balanced and accurate between the models, with good results both clean and under attack. A pattern here is that FARE might be better in some clean accuracy but quickly loses out on robust evaluation. We note that the images and prompts for downstream tasks can be considered zero-shot data, which the models are not explicitly trained on. This aligns with our observation that image-invariance experts, like FARE, thrive in this setting, which can explain FARE's strong performance compared to its image classification results. This suggests that CARE is able to capitalize on FARE's strong specialization, while the image-text alignment knowledge is transferable across text encoders, since we use different text encoders for downstream tasks.

\input{tables/harmonization_comparison}

\subsection{Analysis of harmonization method}
\label{sub:harmonization_analysis}

Table~\ref{tab:harmonization_comparison} shows the zero-shot image classification results of the three models---TeCoA\textsuperscript{4}, FARE\textsuperscript{4}, and CARE\textsuperscript{4}, along with the checkpoints of the two experts at the end of training. The results of zero-shot image classification are surprising: there is an overall increase for the Invariance expert on both in-distribution and zero-shot accuracy compared to FARE. Meanwhile, the Alignment expert, against TeCoA, gets a boost in in-distribution accuracy but also reduced performance on the zero-shot datasets.

We hypothesize this is due to the stronger image-text feature alignment on the Invariance expert, allowing it to perform better overall. Furthermore, since the objective of the Invariance expert (as well as FARE) is to minimize differences between clean and adversarial images, it has not been fine-tuned on the image classification task. The information exchange with the Alignment expert helps the model maintain task knowledge, which is ignored entirely during FARE training.

For the Alignment expert, however, by aligning its image embedding with that of the Invariance expert, the model gets a more robust image embedding and is able to generate even stronger image-text logits for training, culminating in more severe overfitting to the in-distribution dataset. We believe the class logits are dependent on the image embedding, so for the Alignment expert, modifications to the image embedding affect the logits, and by extension, the adversarial training process, directly. The reverse is not true, and the knowledge the Invariance expert gained on the logits does not explicitly affect the loss function on the image embedding.

\begin{figure}[t]                    
  \centering
  \includegraphics[width=\linewidth]{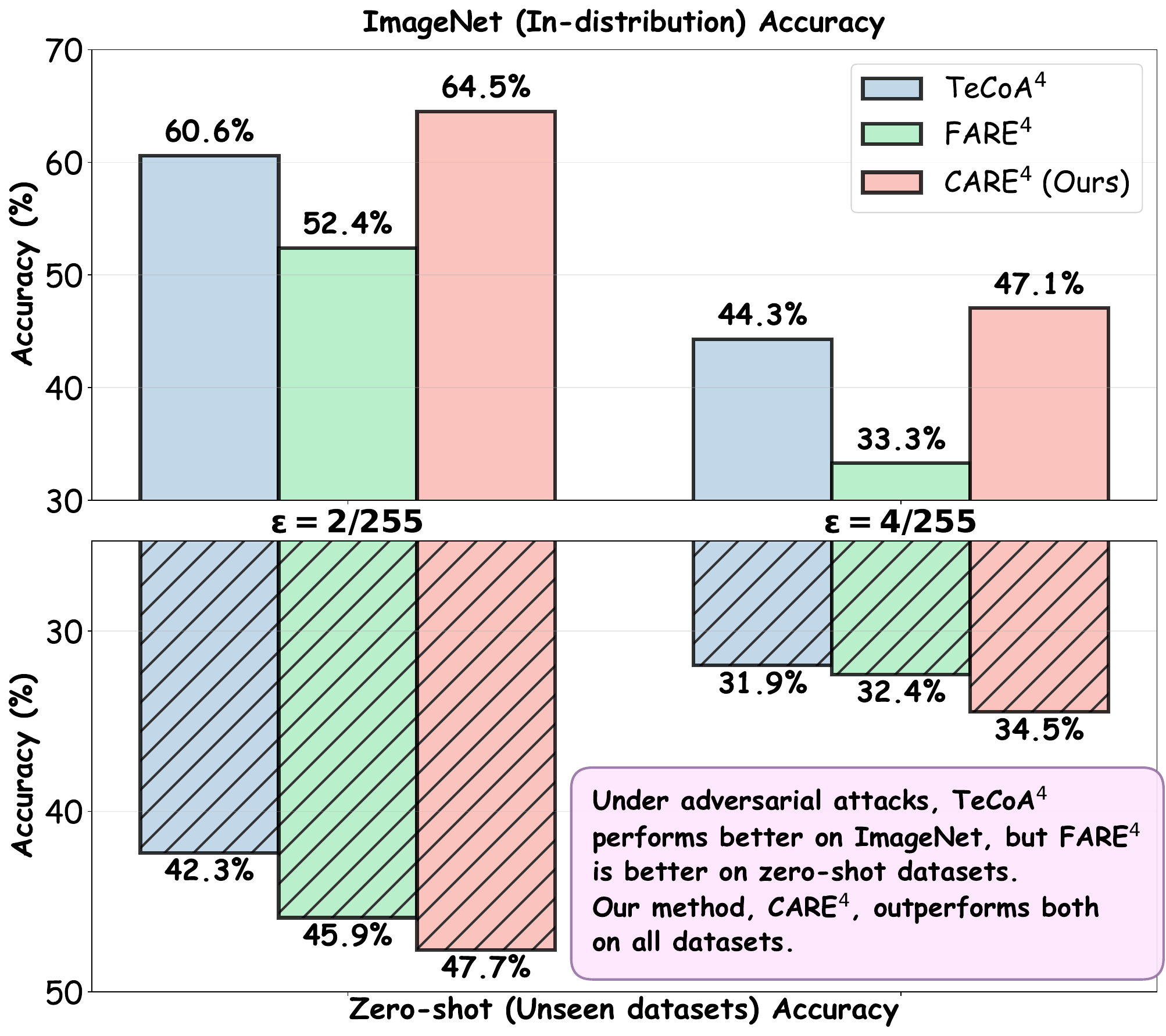}
  \caption{A comparison of robust accuracy on image classification for different fine-tuning methods under adversarial attack $\epsilon=\{2,4\}/255$. Each expert, TeCoA or FARE, specializes in one evaluation setting: TeCoA excels on ImageNet evaluation while FARE performs better outside of ImageNet. Our method, CARE, displays the best result in both cases, showcasing its ability to generalize between robust model experts.}
  \label{fig:robustness_tradeoff}
\end{figure}

\section{Computational resources analysis}
\label{sec:time_analysis}
One of the main bottlenecks in training multiple models like CARE is the additional resource cost. In theory, the added cost of CARE can be calculated as follows:

\begin{itemize}
    \item Double training time due to training two experts at once.
    \item Triple memory usage due to maintaining 2 experts and the merged model at the same time.
\end{itemize}

However, in this section we show that with proper engineering techniques and analysis, this added cost can be partially alleviated for efficient training. The full comparison across methods is shown in Table~\ref{tab:training_efficiency}.

\subsection{Time analysis}
CARE uses 5 PGD steps for GAIP and 5 PGD steps for training each expert. From the perspective of each expert, it still completes a total of 10 PGD steps. However, since the GAIP steps are shared, the total number of PGD steps to perform is only 15. Thus, our training time only increases by a factor of 1.5 instead of 2.

\input{tables/time_comparison}
\input{tables/ablation_harmonization}

\subsection{Memory consumption}
Memory usage comes from both data and models. We notice that each expert creates its own adversarial data for each step, which the other expert does not use. Thus, we delete the adversarial data after each training step. For example, when training the Alignment expert we delete the Invariance expert training data, and vice versa. This method substantially reduces memory consumption from data usage to the same as training a single model.

We also notice that the final merged model CARE is created using EMA, without any gradient update. Thus it can be kept in evaluation mode to further reduce memory. Another extreme optimization would be to offload it to the CPU completely; however, during training we find that putting the model into evaluation mode is enough for our setup.

Therefore the memory consumption of CARE is only equal to that of training both TeCoA and FARE on the same machine despite actually optimizing 3 models. Furthermore, during inference we only use the final CARE model, so there is no added memory usage. 

\section{Ablation study}

In this section we go over various ablation choices when empirically testing our method. Due to time and resource constraints, these ablations were done on ViT-B/32 instead of ViT-L/14. After ablation, the best design (loss function, combination mechanism) is chosen to be integrated, leading to our results shown in Section~\ref{sec:results}.

\subsection{Ablation of expert harmonization loss}
\label{appendix:ablation_loss}
We investigate our choice of loss function for expert harmonization: cosine similarity against KL divergence. The models are trained with $\epsilon=\tfrac{4}{255}$. All other training configurations are kept the same. The results are shown in Table~\ref{tab:ablation_harmonization}, indicating zero-shot image classification accuracy. There are minor differences between the two models, which we regard as insignificant. We ultimately use cosine similarity loss for our ViT-L/14 training, for two reasons:

\begin{itemize}
    \item Theoretical analysis in Theorem~\ref{thm:branch_momentum} shows that cosine similarity acts as momentum to prevent divergence.
    \item Cosine similarity is more numerically stable, as the loss is bounded in the range $[-1, 1]$.
\end{itemize}

\input{tables/ablation_gaip}
\input{tables/framework_comparison}

\subsection{Ablation on GAIP}
\label{sec:gaip_ablation}
We test our GAIP design in 3 different settings: 
\begin{itemize}
    \item Our GAIP using cosine similarity loss.
    \item No GAIP at all (in this case we allocate 5 PGD steps for GAIP into the training process, thus training for 10 steps of PGD).
    \item GAIP using the expert losses directly for PGD (Alignment expert uses cross-entropy loss and Invariance expert uses L2 loss).
\end{itemize}

The results are presented in Table~\ref{tab:gaip_ablation}. Cosine similarity achieves the best result overall. Using the expert losses directly provides a marginal gain in clean accuracy ($\sim$0.5\%) but a substantial drop in robust accuracy. Meanwhile, without GAIP the clean accuracy suffers. Thus, we use cosine similarity in our final version.

\begin{figure*}[t]
  \centering
  {%
    \includegraphics[width=0.3\linewidth]{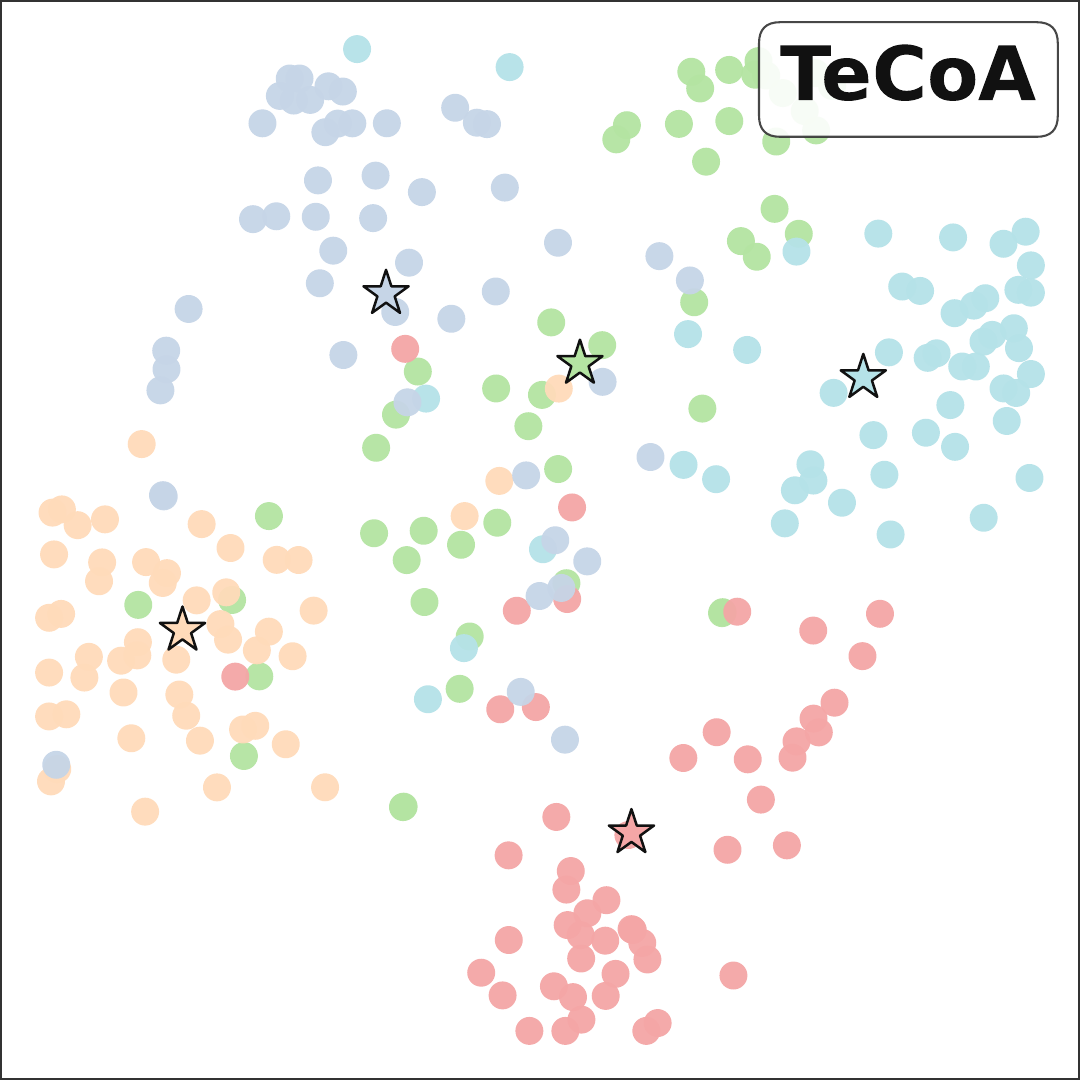}
    \label{fig:tsne_imagenet_tecoa}
  }
  \hfil
  {%
    \includegraphics[width=0.3\linewidth]{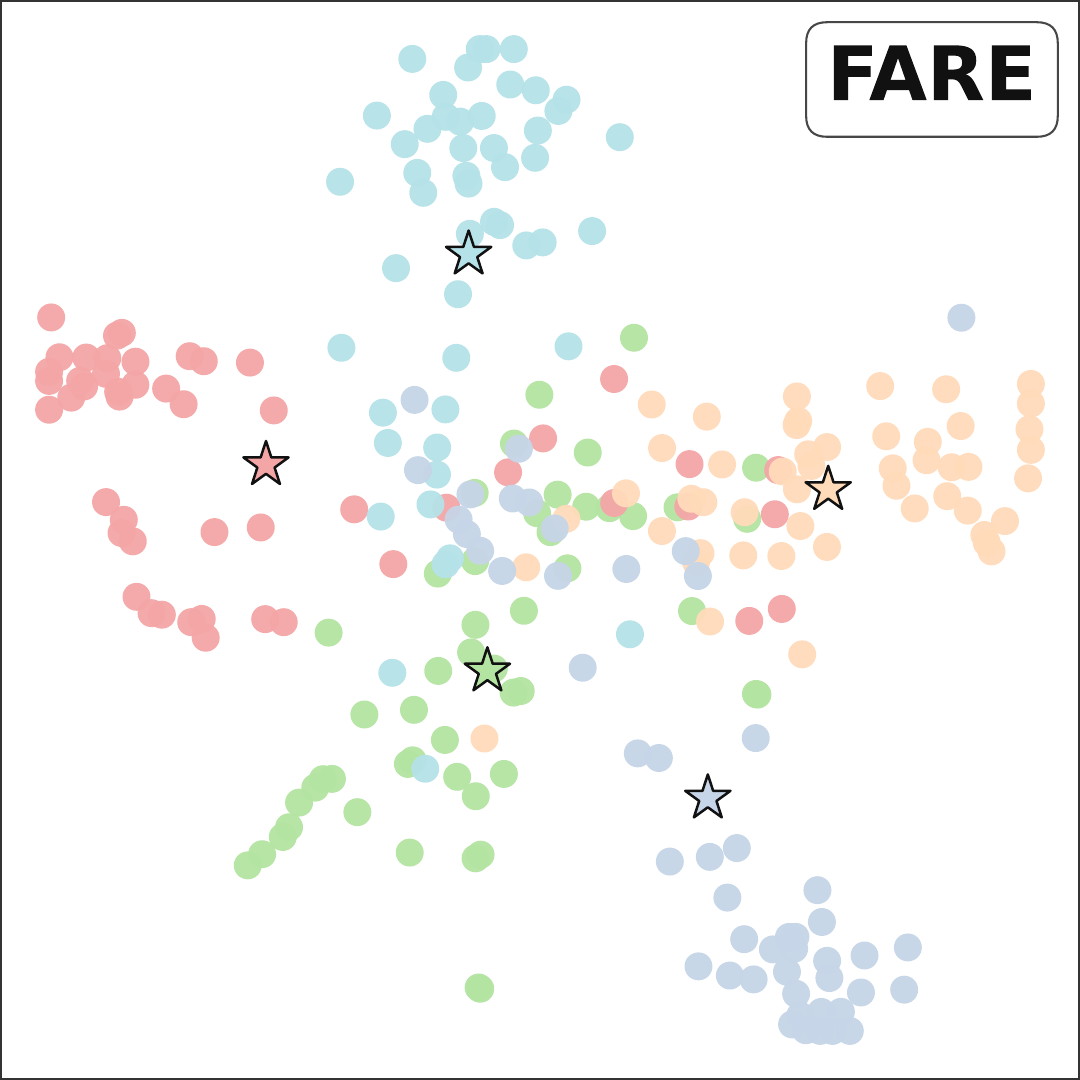}
    \label{fig:tsne_imagenet_fare}
  }
  \hfil
  {%
    \includegraphics[width=0.3\linewidth]{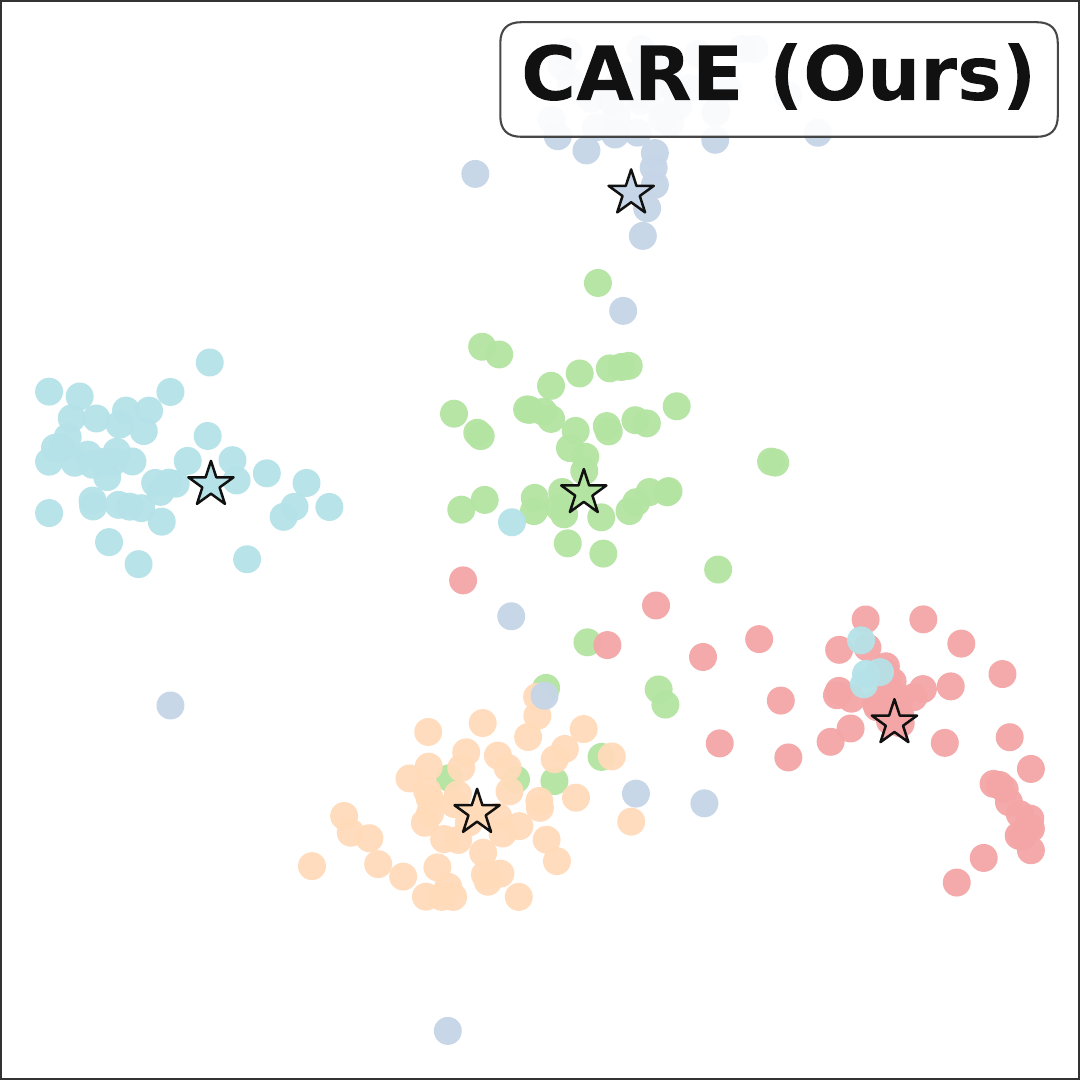}
    \label{fig:tsne_imagenet_care}
  }
  \caption{t-SNE visualization of features from 5 classes in ImageNet.}
  \label{fig:tsne_imagenet}
\end{figure*}

\begin{figure}[!t]                    
  \centering
  \includegraphics[width=\linewidth]{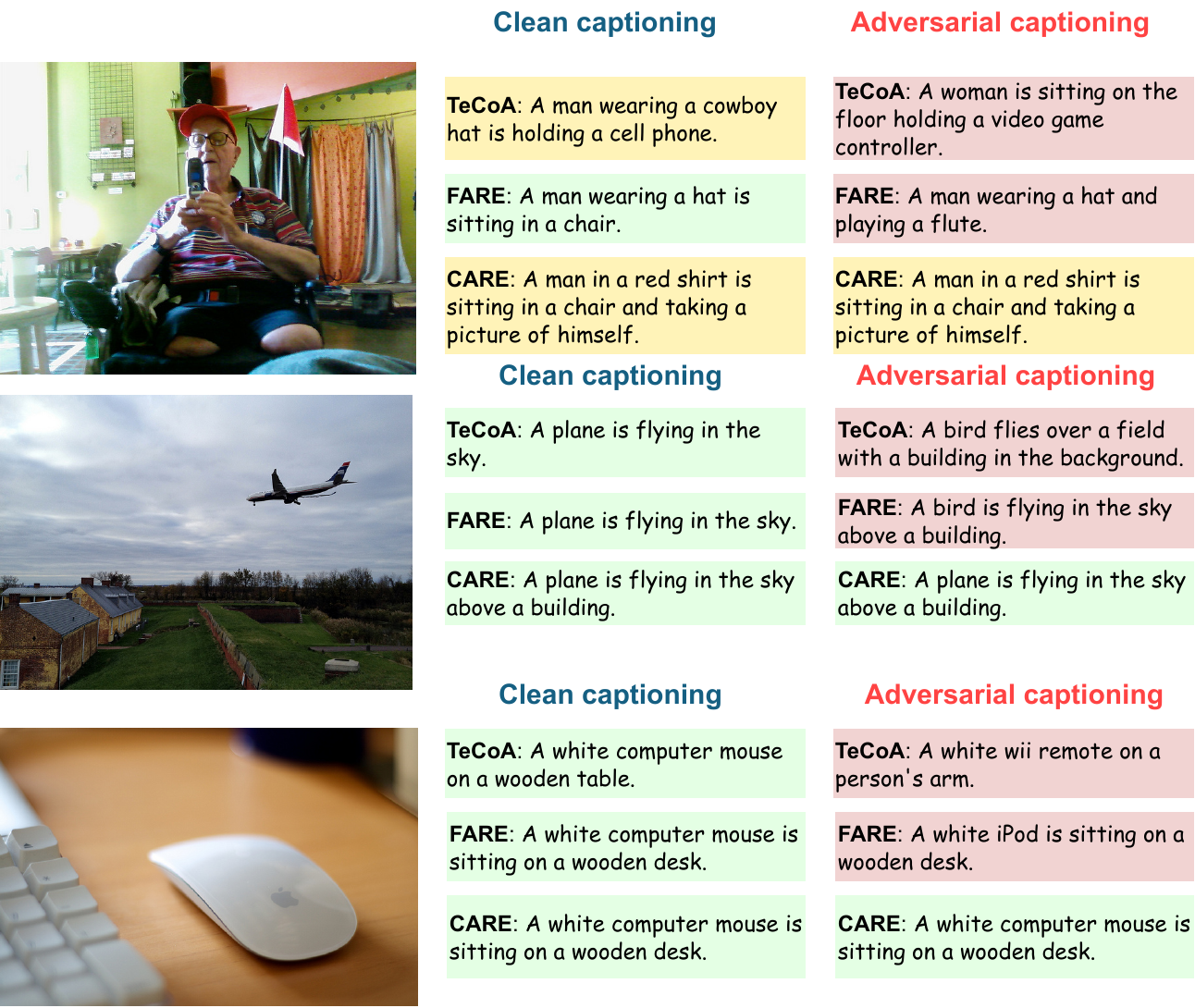}
  \caption{Qualitative results for image captioning using LLaVA. We show the captions generated by each CLIP-based model for the original image and under adversarial attack ($\epsilon = \tfrac{4}{255}$).}
  \label{fig:image_captioning}
\end{figure}

\begin{figure}[!t]                    
  \centering
  \includegraphics[width=\linewidth]{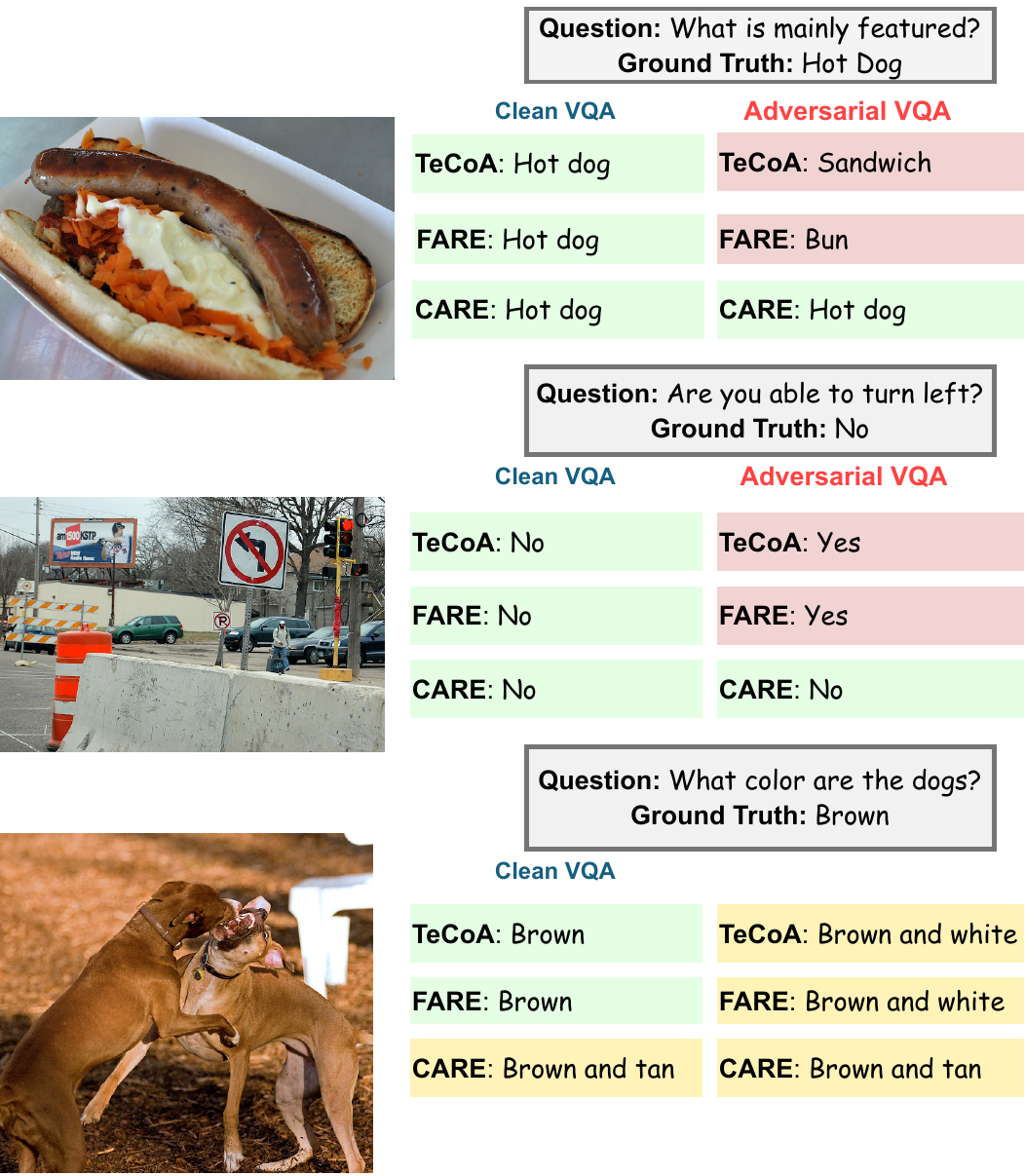}
  \caption{Qualitative results for visual question answering using LLaVA. The adversarial strength used is $\epsilon = \tfrac{4}{255}$. Here the ground truth is denoted by the most frequent answers among annotations for each sample.}
  \label{fig:question_answering}
\end{figure}

\subsection{Comparison to other collaborative frameworks}
\label{appendix:collaborative_comparison}
A direct inspiration for this work comes from \cite{wang2023generalist}. In that paper, the authors trained two experts, one for clean accuracy and the other for adversarial accuracy. However, there is no communication between the experts during training, and the authors noted that to prevent divergence, it is necessary to perform parameter redistribution periodically. Specifically, this involves overwriting the experts' parameters with those of the unified model. We believe this limits the experts' ability for exploration, since they can only optimize in the immediate parameter space near the unified model before being overwritten. We test our expert harmonization method against this approach to show that our framework can prevent divergence while maintaining robustness. In order to do so, we train a ViT-B/32 image encoder with the same setup at $\epsilon=\tfrac{4}{255}$, replacing expert harmonization with parameter redistribution. The redistribution occurs once every 1000 steps, starting from the 8000\textsuperscript{th} step (out of 20000 total steps). The results can be found in Table~\ref{tab:framework_comparison}. It shows that the expert harmonization model outperforms its redistribution counterpart, with up to nearly 2 percentage points more robust accuracy on $\epsilon = \tfrac{4}{255}$.

\section{Visualization}
\label{sec:visualization}

Figure~\ref{fig:tsne_imagenet} presents the t-SNE visualization of feature representations from 5 randomly sampled classes from the ImageNet validation set (50 samples each). We find that TeCoA produces clusters with high spread, making it hard to find the boundary between classes. On the other hand, the classes in FARE are fragmented, with each cluster composed of many mini-clusters, perhaps due to the lack of text supervision to bind images in the same class. CARE is most promising, with compact clusters and well-defined borders.

We present Figure~\ref{fig:image_captioning} and Figure~\ref{fig:question_answering} to show qualitative comparisons on image captioning and visual question answering tasks. Here, green denotes captions/answers that fit the ground truth completely, yellow denotes partially correct answers, while red ones are wrong outputs. CARE is strong against adversarial attacks on vision-language tasks, maintaining consistent answers and captions under perturbations.

\section{Limitations and discussion}
\label{sec:limitations}

This work stems from the observation that adversarial training for VLMs can be primarily separated into two different categories: image-text alignment or image-invariance. We find each of these specializations has its own preferred evaluation settings where it performs best. CARE is thus created to generalize across these robust model experts. In this sense, the construction of CARE is clear and simple: the image-text alignment experts train the logits between image and text, while the image-invariance experts train the image features. Expert harmonization between them is conducted to let each expert learn from the strong embedding of the other. However, other kinds of categorization might exist, in which case the expert harmonization process must be adapted accordingly.

Our work also reveals that a training scheme that optimizes both the image features and image-text alignment does not necessarily yield a better model, as evidenced by the Alignment expert overfitting onto the fine-tuned dataset in Section~\ref{sub:harmonization_analysis}. With regard to the task of image classification, there is an asymmetry between image embeddings and class logits: the logits depend on the embeddings, but not vice versa. 

For future work, we would like to design collaborative methods that account for this asymmetry. Furthermore, while we have done our best to limit the computational cost of CARE, as stated in Section~\ref{sec:time_analysis}, we believe it can still be a bottleneck for larger models, which we want to resolve.

\section{Conclusion}
In this paper, we presented CARE, a novel framework for adversarial training that can generalize across robust model experts. We tested CARE using two state-of-the-art methods: TeCoA and FARE, each of them representing a different approach to adversarial fine-tuning.  Our empirical experiments show that CARE outperforms the methods it is based on. CARE is able to combine the complementary strengths of TeCoA and FARE and gain the specialization of both.

\section*{Acknowledgments}
We declare that our usage of Large Language Model (LLM) (\textit{e.g.,} Claude, ChatGPT) in this work is only for the following: grammar checking, wording, LaTeX phrasing, proof reading and code debugging. \textbf{No LLM was used for the generation of ideas, methods, analyses, or results.} All technical content and claims were authored and verified by us. Outputs were reviewed and edited by the authors, and all citations/equations were checked manually. No proprietary or non-public data is provided to the model.

\bibliographystyle{IEEEtran}
\bibliography{main}


 




\vfill

\end{document}

%% file: tables/baselines_b32.tex
\setlength{\tabcolsep}{3.2pt}
\begin{table*}[htb]
    \centering
    \caption{Image classification accuracy (\%) of ViT-B/32 models trained at $\varepsilon=4/255$. Results are shown for clean images and $\ell_{\infty}$ adversarial examples with $\varepsilon=4/255$. We report in-distribution (ID) performance on ImageNet, zero-shot (ZS) average over 13 unseen datasets, and average of ID and ZS. Best results are shown in \textbf{bold} and \underline{underline} denotes second-best accuracies.}
    \label{tab:baselines_b32}
    \fontsize{8.5}{10}\selectfont

    \begin{tabular}{l|c|ccccccccccccc|c|c}
        \toprule
        & \multicolumn{1}{c|}{\textbf{In-distribution}}
        & \multicolumn{14}{c|}{\textbf{Zero-shot (Unseen datasets)}}
        & \multicolumn{1}{c}{\textbf{Overall average}} \\
        \cmidrule(lr){2-2}
        \cmidrule(lr){3-16}

        \textbf{Method}
        & ImageNet
        & \rotatebox{90}{CalTech}
        & \rotatebox{90}{Cars}
        & \rotatebox{90}{CIFAR10}
        & \rotatebox{90}{CIFAR100}
        & \rotatebox{90}{DTD}
        & \rotatebox{90}{EuroSAT}
        & \rotatebox{90}{FGVC}
        & \rotatebox{90}{Flowers}
        & \rotatebox{90}{ImageNet-R}
        & \rotatebox{90}{ImageNet-S}
        & \rotatebox{90}{PCAM}
        & \rotatebox{90}{OxfordPets}
        & \rotatebox{90}{STL-10}
        & ZS Avg
        & \\
        \midrule

        \multicolumn{17}{l}{\textbf{clean}} \\
        \cmidrule(lr){1-17}
        TeCoA$^4$ & 57.8 & 69.2 & 6.3 & 75.5 & 40.8 & 20.4 & 16.2 & 2.1 & 15.5 & 46.3 & 28.9 & \underline{49.8} & 70.7 & 87.4 & 40.7 & 49.2 \\
        FARE$^4$ & 52.7 & \textbf{78.7} & \textbf{28.2} & 73.6 & \underline{49.2} & \underline{30.7} & \underline{18.3} & \textbf{8.8} & \textbf{31.2} & \underline{50.6} & \underline{29.8} & \textbf{50.0} & \underline{78.6} & \underline{90.4} & \underline{47.6} & \underline{50.1} \\
        PMG$^4$ & 57.0 & 75.2 & 8.5 & \underline{78.4} & 44.1 & 21.6 & 17.3 & 3.0 & 17.6 & 48.1 & 29.2 & \underline{49.8} & 73.1 & 88.3 & 42.6 & 49.8 \\
        TGA$^4$ & \underline{57.9} & 69.5 & 6.7 & 75.5 & 40.4 & 20.3 & 17.9 & 2.1 & 15.8 & 46.4 & 28.9 & \underline{49.8} & 71.0 & 87.3 & 40.9 & 49.4 \\
        \rowcolor{green!15}
        \textbf{CARE$^4$ (Ours)} & \textbf{61.1} & \underline{78.0} & \underline{23.2} & \textbf{83.4} & \textbf{53.9} & \textbf{31.0} & \textbf{20.0} & \underline{6.9} & \underline{31.0} & \textbf{54.8} & \textbf{34.8} & \textbf{50.0} & \textbf{79.2} & \textbf{92.7} & \textbf{49.1} & \textbf{55.1} \\

        \midrule
        \midrule
        \multicolumn{17}{l}{\textbf{$\ell_{\infty}, \varepsilon = 4/255$}} \\
        \cmidrule(lr){1-17}
        TeCoA$^4$ & \underline{28.6} & 53.3 & 0.9 & 32.4 & 17.6 & 9.6 & 0.9 & 0.3 & 4.2 & 20.7 & 13.2 & 46.8 & 37.8 & 61.6 & 23.0 & \underline{25.8} \\
        FARE$^4$ & 19.1 & 54.8 & \textbf{2.8} & 30.2 & 16.8 & \underline{13.2} & \textbf{11.0} & \textbf{0.6} & \underline{4.8} & 16.0 & 10.4 & \textbf{50.2} & 31.3 & \underline{64.0} & 23.5 & 21.3 \\
        PMG$^4$ & 27.4 & \underline{54.9} & 0.8 & \textbf{34.6} & \underline{18.7} & 10.5 & \underline{2.9} & 0.2 & 4.3 & 20.3 & \underline{14.1} & 48.2 & \textbf{39.1} & 62.9 & \underline{24.0} & 25.7 \\
        TGA$^4$ & \textbf{28.8} & 53.1 & 0.9 & 32.9 & 17.5 & 9.4 & 1.5 & \underline{0.5} & 4.3 & \underline{21.2} & 13.7 & 47.6 & 38.4 & 62.1 & 23.3 & \textbf{26.0} \\
        \rowcolor{green!15}
        \textbf{CARE$^4$ (Ours)} & 25.4 & \textbf{55.8} & \underline{1.9} & \underline{33.1} & \textbf{19.6} & \textbf{13.8} & \textbf{11.0} & \underline{0.5} & \textbf{5.5} & \textbf{21.9} & \textbf{15.4} & \underline{49.9} & \underline{39.0} & \textbf{64.8} & \textbf{25.6} & 25.5 \\
        \bottomrule
    \end{tabular}
\end{table*}

%% file: tables/zeroshot.tex
\setlength{\tabcolsep}{3.2pt}
\begin{table*}[htb]
    \centering
    \caption{Image classification accuracy (\%) for ViT-L/14. Results are shown for three evaluation settings: clean images, $\ell_{\infty}$ adversarial examples with perturbation size $\varepsilon=2/255$, and $\varepsilon=4/255$. We report in-distribution (ID) performance on ImageNet, zero-shot (ZS) average, and average of ID and ZS. \textbf{Bold} and \underline{underline} denotes best and second-best results.}
    \label{tab:zero_shot_results}
    \fontsize{8.5}{10}\selectfont

    \begin{tabular}{ll|c|ccccccccccccc|c|c}
        \toprule
        \multicolumn{2}{c|}{}
        & \multicolumn{1}{c|}{\textbf{In-distribution}}
        & \multicolumn{14}{c|}{\textbf{Zero-shot (Unseen datasets)}}
        & \multicolumn{1}{c}{\textbf{Overall average}} \\
        \cmidrule(lr){3-3}
        \cmidrule(lr){4-17}
        
        \multicolumn{2}{l|}{\textbf{Method}}
        & ImageNet
        & \begin{tabular}[t]{@{}c@{}}\rotatebox{90}{CalTech}\end{tabular}
        & \begin{tabular}[t]{@{}c@{}}\rotatebox{90}{Cars}\end{tabular}
        & \begin{tabular}[t]{@{}c@{}}\rotatebox{90}{CIFAR10}\end{tabular}
        & \begin{tabular}[t]{@{}c@{}}\rotatebox{90}{CIFAR100}\end{tabular}
        & \begin{tabular}[t]{@{}c@{}}\rotatebox{90}{DTD}\end{tabular}
        & \begin{tabular}[t]{@{}c@{}}\rotatebox{90}{EuroSAT}\end{tabular}
        & \begin{tabular}[t]{@{}c@{}}\rotatebox{90}{FGVC}\end{tabular}
        & \begin{tabular}[t]{@{}c@{}}\rotatebox{90}{Flowers}\end{tabular}
        & \begin{tabular}[t]{@{}c@{}}\rotatebox{90}{ImageNet-R}\end{tabular}
        & \begin{tabular}[t]{@{}c@{}}\rotatebox{90}{ImageNet-S}\end{tabular}
        & \begin{tabular}[t]{@{}c@{}}\rotatebox{90}{PCAM}\end{tabular}
        & \begin{tabular}[t]{@{}c@{}}\rotatebox{90}{OxfordPets}\end{tabular}
        & \begin{tabular}[t]{@{}c@{}}\rotatebox{90}{STL-10}\end{tabular}
        & ZS Avg
        &  \\
        \midrule
        
        \multicolumn{18}{l}{\textbf{clean}} \\
        \cmidrule(lr){1-18}
        \rowcolor{gray!15}
        \multicolumn{2}{l|}{CLIP} 
        & 74.9 
        & 83.3 & 77.9 & 95.2 & 71.1 & 55.2 & 62.6 & 31.8 & 79.2 & 87.9 & 59.6 & 52.0 & 93.2 & 99.3 
        & 73.1 
        & 74.0 \\

        \multicolumn{2}{l|}{TeCoA$^2$-CLIP} 
        & \underline{80.2} 
        & 80.7 & 50.1 & 87.5 & 60.7 & 44.4 & \textbf{26.1} & 14.0 & 51.8 & 80.1 & 58.4 & 49.9 & 80.0 & 96.1 
        & 60.0 
        & 70.1 \\

        \multicolumn{2}{l|}{FARE$^2$-CLIP} 
        & 74.2 
        & \textbf{84.8} & \textbf{70.5} & \underline{89.5} & \underline{69.1} & \textbf{50.0} & 25.4 & \textbf{26.7} & \textbf{70.6} & \textbf{85.5} & \underline{59.7} & \underline{50.0} & \textbf{91.1} & \textbf{98.5} 
        & \textbf{67.0} 
        & \underline{70.6} \\

        \rowcolor{green!15}
        \multicolumn{2}{l|}{\textbf{CARE$^2$-CLIP (Ours)}} 
        & \textbf{80.7} 
        & \underline{83.9} & \underline{54.3} & \textbf{93.3} & \textbf{71.6} & \underline{47.6} & \underline{25.6} & \underline{16.1} & \underline{53.9} & \underline{83.8} & \textbf{61.3} & \textbf{50.6} & \underline{86.1} & \underline{98.4} 
        & \underline{63.6}  
        & \textbf{72.2} \\
        
        \midrule

        \multicolumn{2}{l|}{TeCoA$^4$-CLIP} 
        & \underline{75.2} 
        & 78.4 & 37.9 & \underline{79.6} & 50.3 & 38.0 & \underline{22.5} & 11.8 & 38.4 & 74.3 & 54.2 & \underline{50.0} & 76.1 & 93.4 
        & 54.2 
        & 64.7 \\

        \multicolumn{2}{l|}{FARE$^4$-CLIP} 
        & 70.4 
        & \textbf{84.7} & \textbf{63.8} & 77.7 & \underline{56.5} & \underline{43.8} & 18.3 & \textbf{22.0} & \textbf{58.1} & \underline{80.2} & \underline{56.7} & \underline{50.0} & \textbf{87.1} & \underline{96.0} 
        & \underline{61.1} 
        & \underline{65.8} \\

        \rowcolor{green!15}

        \multicolumn{2}{l|}{\textbf{CARE$^4$-CLIP (Ours)}} 
        & \textbf{78.1} 
        & \underline{82.7} & \underline{48.9} & \textbf{92.2} & \textbf{70.0} & \textbf{47.7} & \textbf{25.0} & \underline{14.0} & \underline{48.8} & \textbf{82.2} & \textbf{60.3} & \textbf{55.4} & \underline{83.5} & \textbf{98.1} 
        & \textbf{62.2}  
        & \textbf{70.2} \\
        
        \midrule
        \midrule
        \multicolumn{18}{l}{\textbf{$\ell_{\infty}, \varepsilon = 2/255$}} \\
        \cmidrule(lr){1-18}

        \rowcolor{gray!15}
        \multicolumn{2}{l|}{CLIP} 
        & 0.0 
        & 0.0 & 0.0 & 0.0 & 0.0 & 0.0 & 0.0 & 0.0 & 0.0 & 0.0 & 0.1 & 0.0 & 0.0 & 0.0 
        & 0.0 
        & 0.0 \\

        \multicolumn{2}{l|}{TeCoA$^2$-CLIP} 
        & \underline{62.3} 
        & 70.2 & 22.2 & \underline{63.7} & 35.0 & \underline{27.0} & \underline{12.8} & 5.8 & \underline{27.6} & \underline{58.8} & \underline{45.2} & 40.0 & \underline{69.7} & 88.7 
        & \underline{43.6} 
        & \underline{53.0} \\

        \multicolumn{2}{l|}{FARE$^2$-CLIP} 
        & 46.1 
        & \underline{73.0} & \underline{26.0} & 60.3 & \underline{35.6} & 26.7 & 6.2 & \underline{5.9} & \textbf{31.2} & 56.5 & 38.3 & \underline{41.9} & 68.3 & \underline{90.1} 
        & 43.1  
        & 44.6 \\

        \rowcolor{green!15}
        \multicolumn{2}{l|}{\textbf{CARE$^2$-CLIP (Ours)}} 
        & \textbf{64.2}  
        & \textbf{74.3} & \textbf{26.6} & \textbf{73.5} & \textbf{44.9} & \textbf{30.3} & \textbf{13.5} & \textbf{6.9} & 27.5 & \textbf{65.1} & \textbf{48.9} & \textbf{47.7} & \textbf{73.6} & \textbf{92.7} 
        & \textbf{48.1} 
        & \textbf{56.2} \\

        \midrule

        \multicolumn{2}{l|}{TeCoA$^4$-CLIP} 
        & \underline{60.6} 
        & 69.7 & 17.9 & \underline{59.7} & 33.7 & 26.5 & 8.0 & 5.0 & 24.1 & 59.2 & \underline{43.0} & \underline{48.8} & 68.0 & 86.7 
        & 42.3 
        & \underline{51.5} \\

        \multicolumn{2}{l|}{FARE$^4$-CLIP} 
        & 52.4 
        & \textbf{76.7} & \textbf{30.0} & 57.3 & \underline{36.5} & \underline{28.3} & \underline{12.8} & \textbf{8.2} & \textbf{31.3} & \underline{61.6} & 41.6 & \textbf{50.2} & \underline{72.4} & \underline{89.6} 
        & \underline{45.9} 
        & 49.2 \\

        \rowcolor{green!15}
        \multicolumn{2}{l|}{\textbf{CARE$^4$-CLIP (Ours)}} 
        & \textbf{64.5}  
        & \underline{74.5} & \underline{23.9} & \textbf{71.2} & \textbf{46.1} & \textbf{32.2} & \textbf{15.1} & \underline{6.7} & \underline{24.5} & \textbf{64.0} & \textbf{48.1} & 47.9 & \textbf{73.0} & \textbf{92.6} 
        & \textbf{47.7} 
        & \textbf{56.1} \\

        \midrule
        \midrule
        \multicolumn{18}{l}{\textbf{$\ell_{\infty}, \varepsilon = 4/255$}} \\
        \cmidrule(lr){1-18}

        \rowcolor{gray!15}
        \multicolumn{2}{l|}{CLIP} 
        & 0.0 
        & 0.0 & 0.0 & 0.0 & 0.0 & 0.0 & 0.0 & 0.0 & 0.0 & 0.0 & 0.0 & 0.0 & 0.0 & 0.0 
        & 0.0 
        & 0.0 \\

        \multicolumn{2}{l|}{TeCoA$^2$-CLIP} 
        & \underline{37.3} 
        & \underline{57.4} & \underline{6.5} & \underline{31.0} & \underline{17.8} & \underline{14.7} & \underline{7.7} & \underline{1.1} & \textbf{9.8} & \underline{36.7} & \underline{32.8} & 16.0 & \underline{50.3} & \underline{69.2} 
        & \underline{27.0} 
        & \underline{32.2} \\

        \multicolumn{2}{l|}{FARE$^2$-CLIP} 
        & 16.6 
        & 46.6 & 4.8 & 25.9 & 13.9 & 11.7 & 0.5 & 0.6 & 7.1 & 25.6 & 22.5 & \underline{17.2} & 27.9 & 61.7 
        & 20.5  
        & 18.6 \\

        \rowcolor{green!15}
        \multicolumn{2}{l|}{\textbf{CARE$^2$-CLIP (Ours)}} 
        & \textbf{40.3}  
        & \textbf{60.6} & \textbf{7.3} & \textbf{38.0} & \textbf{21.5} & \textbf{17.3} & \textbf{8.1} & \textbf{2.0} & \underline{9.6} & \textbf{40.5} & \textbf{35.7} & \textbf{26.2} & \textbf{53.0} & \textbf{73.4} 
        & \textbf{30.3} 
        & \textbf{35.3} \\

        \midrule

        \multicolumn{2}{l|}{TeCoA$^4$-CLIP} 
        & \underline{44.3} 
        & 60.9 & \underline{8.4} & \underline{37.1} & \underline{21.5} & 16.4 & 6.6 & 2.1 & \underline{12.4} & \underline{41.9} & \underline{34.2} & \underline{44.0} & \underline{55.2} & 74.3 
        & 31.9 
        & \underline{38.1} \\

        \multicolumn{2}{l|}{FARE$^4$-CLIP} 
        & 33.3 
        & \underline{64.1} & \textbf{12.7} & 34.6 & 20.2 & \underline{17.3} & \underline{11.1} & \underline{2.6} & \textbf{12.5} & 40.6 & 30.9 & \textbf{50.2} & 50.7 & \underline{74.4} 
        & \underline{32.4}  
        & 32.9 \\

        \rowcolor{green!15}
        \multicolumn{2}{l|}{\textbf{CARE$^4$-CLIP (Ours)}} 
        & \textbf{47.1} 
        & \textbf{64.6} & \underline{8.4} & \textbf{43.3} & \textbf{26.0} & \textbf{20.3} & \textbf{11.6} & \textbf{3.1} & 12.3 & \textbf{46.7} & \textbf{38.7} & 34.4 & \textbf{59.0} & \textbf{79.8} 
        & \textbf{34.5} 
        & \textbf{40.8} \\
        \bottomrule
    \end{tabular}
\end{table*}

%% file: tables/image_captioning.tex
\setlength{\tabcolsep}{7pt}
\begin{table*}[htbp]
    \centering
    \small
    \caption{
    Image captioning performance with LLaVA~1.5-7B as the language head.
    CIDEr score is reported under clean and adversarial settings
    (\(\ell_\infty\in\{2,4\}/255\)).
    Best and second-best results are shown in \textbf{bold} and
    \underline{underline}, respectively.
    }
    \label{tab:image_captioning_results}

    \begin{tabular}{l|ccc|ccc|ccc}
        \toprule
        & \multicolumn{9}{c}{\textbf{Image Captioning}} \\
        \cmidrule(lr){2-10}

        \multirow{2}{*}{\textbf{Method}}
        & \multicolumn{3}{c|}{\textbf{COCO}}
        & \multicolumn{3}{c|}{\textbf{Flickr30k}}
        & \multicolumn{3}{c}{\textbf{Average}}
        \\
        \cmidrule(lr){2-4}
        \cmidrule(lr){5-7}
        \cmidrule(lr){8-10}

        &
        clean & {$\frac{2}{255}$} & {$\frac{4}{255}$} &
        clean & {$\frac{2}{255}$} & {$\frac{4}{255}$} &
        clean &
        {$\frac{2}{255}$} &
        {$\frac{4}{255}$}
        \\
        \midrule

        \rowcolor{gray!15}
        CLIP &
        115.5 & 4.0 & 3.1 &
        77.5 & 1.6 & 1.0 &
        96.5 &
        2.8 &
        2.05
        \\
        \midrule

        TeCoA$^{2}$ &
        98.4 & 44.2 & 30.3 &
        57.1 & 23.2 & 15.3 &
        77.8 &
        33.7 &
        22.8
        \\

        FARE$^{2}$ &
        \textbf{109.9} & \underline{53.6} & \underline{31.0} &
        \textbf{71.1} & \underline{29.5} & \underline{17.5} &
        \textbf{90.5} &
        \underline{41.6} &
        \underline{24.3}
        \\

        \rowcolor{green!15}

        \textbf{CARE$^{2}$} &
        \underline{107.8} & \textbf{55.6} & \textbf{33.6} &
        \underline{66.6} & \textbf{30.8} & \textbf{19.7} &
        \underline{87.2} &
        \textbf{43.2} &
        \textbf{26.7}
        \\
        \midrule

        TeCoA$^{4}$ &
        88.3 & 50.9 & 35.3 &
        48.6 & 27.9 & 19.5 &
        68.5 &
        39.4 &
        \underline{27.4}
        \\

        FARE$^{4}$ &
        \underline{102.4} & \underline{57.1} & \textbf{40.9} &
        \underline{61.6} & \underline{31.4} & \underline{22.8} &
        \underline{82.0} &
        \underline{44.3} &
        \textbf{31.9}
        \\

        \rowcolor{green!15}

        \textbf{CARE$^{4}$} &
        \textbf{103.6} & \textbf{59.2} & \underline{40.5} &
        \textbf{64.8} & \textbf{33.0} & \textbf{23.2} &
        \textbf{84.2} &
        \textbf{46.1} &
        \textbf{31.9}
        \\
        \bottomrule
    \end{tabular}
\end{table*}

%% file: tables/vqa.tex
\setlength{\tabcolsep}{7pt}
\begin{table*}[htbp]
    \centering
    \small
    \caption{
    Performance on visual question answering benchmarks with LLaVA~1.5-7B as the language head.
    VQA accuracy is reported under clean and adversarial settings
    (\(\ell_\infty\in\{2,4\}/255\)).
    Best results are shown in \textbf{bold} and
    \underline{underline} denotes second-best accuracies.
    }
    \label{tab:vqa_results}

    \begin{tabular}{l|ccc|ccc|ccc}
        \toprule
        & \multicolumn{9}{c}{\textbf{Visual Question Answering}} \\
        \cmidrule(lr){2-10}

        \multirow{2}{*}{\textbf{Method}}
        & \multicolumn{3}{c|}{\textbf{TextVQA}}
        & \multicolumn{3}{c|}{\textbf{VQAv2}}
        & \multicolumn{3}{c}{\textbf{Average}}
        \\
        \cmidrule(lr){2-4}
        \cmidrule(lr){5-7}
        \cmidrule(lr){8-10}

        &
        clean & {$\frac{2}{255}$} & {$\frac{4}{255}$} &
        clean & {$\frac{2}{255}$} & {$\frac{4}{255}$} &
        clean &
        {$\frac{2}{255}$} &
        {$\frac{4}{255}$}
        \\
        \midrule

        \rowcolor{gray!15}
        CLIP &
        37.1 & 0.5 & 0.0 &
        74.5 & 2.9 & 0.0 &
        55.8 &
        1.7 &
        0.0
        \\
        \midrule

        TeCoA$^{2}$ &
        24.1 & 12.1 & 8.8 &
        \underline{66.9} & 33.8 & 21.8 &
        45.5 &
        23.0 &
        15.3
        \\

        FARE$^{2}$ &
        \textbf{31.9} & \underline{14.7} & \underline{9.1} &
        \textbf{71.7} & \underline{34.9} & \underline{23.0} &
        \textbf{51.8} &
        \underline{24.8} &
        \underline{16.1}
        \\

        \rowcolor{green!15}
        \textbf{CARE$^{2}$} &
        \underline{26.8} & \textbf{15.7} & \textbf{9.5} &
        66.7 & \textbf{43.1} & \textbf{27.6} &
        \underline{46.8} &
        \textbf{29.4} &
        \textbf{18.6}
        \\
        \midrule

        TeCoA$^{4}$ &
        20.7 & \underline{12.6} & 9.3 &
        63.2 & \underline{41.0} & \underline{31.7} &
        42.0 &
        26.8 &
        20.5
        \\

        FARE$^{4}$ &
        \textbf{27.6} & \textbf{15.8} & \underline{10.9} &
        \textbf{68.3} & 40.7 & 30.5 &
        \textbf{48.0} &
        \underline{28.3} &
        \underline{20.7}
        \\

        \rowcolor{green!15}
        \textbf{CARE$^{4}$} &
        \underline{23.7} & \textbf{15.8} & \textbf{11.1} &
        \underline{66.3} & \textbf{46.3} & \textbf{34.8} &
        \underline{45.0} &
        \textbf{31.0} &
        \textbf{23.0}
        \\
        \bottomrule
    \end{tabular}
\end{table*}

%% file: tables/harmonization_comparison.tex
\setlength{\tabcolsep}{3pt}
\begin{table}[!t]
    \centering
    \small
    \renewcommand{\arraystretch}{1.12}
    \caption{Average zero-shot image classification accuracies (\%) of TeCoA\textsuperscript{4}, FARE\textsuperscript{4}, CARE\textsuperscript{4} and the two experts at the end of training.}
    \label{tab:harmonization_comparison}
    \fontsize{8.5}{10}\selectfont 
    
    \begin{tabular}{l|ccc|ccc}
        \toprule
        \multirow{2}{*}{\textbf{Method}} & \multicolumn{3}{c|}{\textbf{In-distribution}} & \multicolumn{3}{c}{\textbf{Zero-shot}} \\
        \cmidrule(lr){2-4}\cmidrule(lr){5-7}
        & {clean} & {$\frac{2}{255}$} & {$\frac{4}{255}$} &
        {clean} & {$\frac{2}{255}$} & {$\frac{4}{255}$}
        \\
        \midrule
        TeCoA$^{4}$ &
        75.2 & 60.6 & 44.3 & 54.2 & 42.3 & 31.9
        \\
        
        Alignment expert (TeCoA-based) &
        77.3 & 65.3 & 49.4 & 50.5 & 39.0 & 28.4
        \\
        \midrule
        FARE$^{4}$  &
        70.4 & 52.4 & 33.3 & 61.1 & 45.9 & 32.4
        \\

        Invariance expert (FARE-based) &
        73.1 & 56.9 & 37.6 & 61.4 & 46.9 & 33.3
        \\
        \midrule
        \rowcolor{green!15}
        CARE$^{4}$ &
        78.1 & 64.5 & 47.1 & 62.2 & 47.7 & 34.5
        \\
        
        \bottomrule
    \end{tabular}
\end{table}

%% file: tables/time_comparison.tex
\begin{table}[!t]
\centering
\caption{Training efficiency comparison among different methods.}
\label{tab:training_efficiency}
\setlength{\tabcolsep}{8pt}
\renewcommand{\arraystretch}{1.15}

\begin{tabular}{l|c|c|c}
\toprule
\textbf{Method} 
& \textbf{Memory} 
& \textbf{Time (1 step)} 
& \textbf{Total time (20000 steps)} \\
\midrule

TeCoA 
& 29.1 GB 
& 8.7 sec 
& 48 hours \\

FARE 
& 25.0 GB 
& 9.0 sec 
& 50 hours \\

CARE 
& 52.6 GB 
& 15.0 sec 
& 83 hours \\

\bottomrule
\end{tabular}
\end{table}

%% file: tables/ablation_harmonization.tex
\setlength{\tabcolsep}{3.0pt}
\begin{table}[!t]
    \centering
    \small
    \renewcommand{\arraystretch}{1.12}
    \caption{Average zero-shot image classification accuracies (\%) of ViT-B/32 CARE\textsuperscript{4} using different harmonization loss functions.}
    \label{tab:ablation_harmonization}
    \fontsize{8.5}{10}\selectfont 
    
    \begin{tabular}{l|ccc|ccc}
        \toprule

        \textbf{Loss function} & \multicolumn{3}{c|}{\textbf{In-distribution}} & \multicolumn{3}{c}{\textbf{Zero-shot (Unseen datasets)}} \\
        \cmidrule(lr){2-4}\cmidrule(lr){5-7}
        & {clean} & {$\frac{2}{255}$} & {$\frac{4}{255}$} &
        
        {clean} & {$\frac{2}{255}$} & {$\frac{4}{255}$}
        \\
        \midrule
        
        KL Divergence &
        61.06 & 42.50 & 25.46 & 49.18 & 36.90 & 25.55
        \\
        
        Cosine Similarity &
        61.06 & 42.36 & 25.40 & 49.14 & 36.86 & 25.55
        \\
        
        \bottomrule
    \end{tabular}
\end{table}

%% file: tables/ablation_gaip.tex
\begin{table}[!t]
\centering
\caption{Ablation study of GAIP under clean and adversarial settings. We show average image classification accuracy for In-distribution and Zero-shot datasets, with and without adversarial attack.}
\label{tab:gaip_ablation}
\setlength{\tabcolsep}{5pt}

\begin{tabular}{l|ccc|ccc}
\toprule
& \multicolumn{3}{c|}{\textbf{Clean (\%)}} 
& \multicolumn{3}{c}{\textbf{Robust ($\epsilon=4/255$) (\%)}} \\
\cmidrule(lr){2-4} \cmidrule(lr){5-7}

\textbf{Method}
& \textbf{ID}
& \textbf{ZS}
& \textbf{Total}
& \textbf{ID}
& \textbf{ZS}
& \textbf{Total} \\
\midrule

Cosine similarity
& 61.06
& 49.14
& 110.20
& 25.40
& 25.55
& \textbf{50.95} \\

No GAIP
& 58.14
& 47.24
& 105.38
& 24.86
& 25.90
& 50.76 \\

L2 and CE losses
& 60.12
& 50.61
& \textbf{110.73}
& 21.48
& 24.12
& 45.60 \\

\bottomrule
\end{tabular}
\end{table}

%% file: tables/framework_comparison.tex
\setlength{\tabcolsep}{2.0pt}
\begin{table}[!t]
    \centering
    \small
    \renewcommand{\arraystretch}{1.12}
    \caption{Average image classification accuracies (\%) of ViT-B/32 CARE\textsuperscript{4} using Expert harmonization and using Parameter redistribution.}
    \label{tab:framework_comparison}
    \fontsize{8.0}{10}\selectfont 
    
    \begin{tabular}{l|ccc|ccc}
        \toprule

        \textbf{Harmonization method} & \multicolumn{3}{c|}{\textbf{ID}} & \multicolumn{3}{c}{\textbf{ZS}} \\
        \cmidrule(lr){2-4}\cmidrule(lr){5-7}
        & {clean} & {$\frac{2}{255}$} & {$\frac{4}{255}$} &
        
        {clean} & {$\frac{2}{255}$} & {$\frac{4}{255}$}
        \\
        \midrule
        Parameter redistribution &
        60.52 & 41.16 & 23.70 & 49.22 & 35.93 & 24.99
        \\
        
        Expert harmonization (Ours) &
        \textbf{61.06} & \textbf{42.36} & \textbf{25.40} & 49.14 & \textbf{36.86} & \textbf{25.55}
        \\
        
        \bottomrule
    \end{tabular}
\end{table}